\documentclass[letterpaper, 10 pt, conference]{ieeeconf}  
\usepackage{algorithmicx}
\usepackage[ruled,linesnumbered]{algorithm2e}
\usepackage{amssymb}
\usepackage{amsmath}
\usepackage{bm}
\usepackage{booktabs}
\usepackage{color}
\usepackage{cite}
\usepackage{graphicx}
\usepackage[hidelinks]{hyperref}
\usepackage{multirow}
 
\usepackage{caption}
\usepackage{subcaption}
\newtheorem{theorem}{Theorem}
\usepackage{threeparttable}
\usepackage{times}
      
\IEEEoverridecommandlockouts                              %
\title{\LARGE \bf
VLSA: Vision-Language-Action Models with Plug-and-Play Safety Constraint Layer
}

\author{Songqiao Hu$^{1,2,*}$, Zeyi Liu$^{1,2,*}$, Shuang Liu$^{1,2,3}$, Jun Cen$^{4}$, Zihan Meng$^{3}$, \\ Shihefeng Wang$^{1}$, Xiang Li$^{1}$, and Xiao He$^{1,2,\dagger}$
\thanks{$^{*}$Equal contribution. $^{\dagger}$Corresponding author (e-mail: {\tt\small hexiao@tsinghua.edu.cn}).}%
\thanks{This work was supported in part by the National Natural Science Foundation of China under Grants 62525308, 624B2087, 62473223, and 52172323, and in part by the Beijing Natural Science Foundation under Grant L241016.}%
\thanks{$^{1}$Department of Automation, Tsinghua University, Beijing 100084, China.}%
\thanks{$^{2}$Institute for Embodied Intelligence and Robotics, Tsinghua University, Beijing 100084, China.}%
\thanks{$^{3}$TetraBOT. $^{4}$DAMO Academy, Alibaba Group.}%
}

\begin{document}

\maketitle
\thispagestyle{empty}
\pagestyle{empty}

\begin{abstract}

Vision-Language-Action (VLA) models have demonstrated remarkable capabilities in generalizing across diverse robotic manipulation tasks. However, deploying these models in unstructured environments remains challenging due to the critical need for simultaneous task compliance and safety assurance, particularly in preventing potential collisions during physical interactions. In this work, we introduce a Vision-Language-Safe Action (VLSA) architecture, named AEGIS, which contains a plug-and-play safety constraint (SC) layer formulated via control barrier functions. AEGIS integrates directly with existing VLA models to improve safety with theoretical guarantees, while maintaining their original instruction-following performance. To evaluate the efficacy of our architecture, we construct a comprehensive safety-critical benchmark SafeLIBERO, spanning distinct manipulation scenarios characterized by varying degrees of spatial complexity and obstacle intervention. Extensive experiments demonstrate the superiority of our method over state-of-the-art baselines. Notably, AEGIS achieves over 50\% improvement in obstacle avoidance rate while substantially increasing the task success rate by nearly 10\%. 
All benchmark datasets, code, and supplementary materials are publicly available at \url{https://vlsa-aegis.github.io/}.
\end{abstract}

\section{Introduction}

Vision-Language-Action (VLA) models have demonstrated remarkable generalization capabilities across a wide range of robotic manipulation tasks by unifying visual encoding, language understanding, and action control into a single end-to-end framework \cite{zitkovich2023rt,sapkota2025vision,kawaharazuka2025vision}. As a key technological pathway linking semantic parsing with embodied control, VLA models enable robots to generate coherent physical actions from visual observations and natural language goals, representing a significant step toward general-purpose embodied intelligence \cite{song2025embodied}. Recent advances, such as $\pi_{0.5}$ \cite{intelligence2025pi}  and OpenVLA \cite{kim2024openvla}, have made substantial progress in spatial reasoning and multimodal alignment, enhancing reasoning efficiency and task execution performance.

Safety stands as a prerequisite for the real-world deployment of VLA models, as collisions in unstructured environments can lead to hardware damage, human injury, or property loss \cite{liu2023real,zhong2025survey,liu2023dynamic}. Although semantic understanding is well performed by existing VLA models, safety guarantees are often overlooked \cite{sapkota2025vision}. In recent works, safety constraints are integrated through reinforcement learning \cite{zhang2025safevla,gu2024review,brunke2022safe}. 
These approaches have demonstrated promising results in specific tasks. 
Nevertheless, these retraining-based methods require prohibitive computational costs and are further bottlenecked by the expensive, labor-intensive nature of collecting safe real-world robotic data. Furthermore, their flexibility is severely restricted when deployed with existing pretrained models \cite{hu2022lora}.
Additionally,  safety is typically handled as a soft objective through reward penalties rather than a hard constraint in such cases \cite{hasanzadezonuzy2021learning}. Consequently, robot behavior in unstructured environments remains overly reliant on the generalization capability of VLA models, and unsafe trajectories might be produced when out-of-distribution scenarios are encountered.

\begin{figure}[htbp]
  \centering
  
  \begin{subfigure}[t]{0.49\linewidth}
    \centering
    \includegraphics[width=\linewidth]{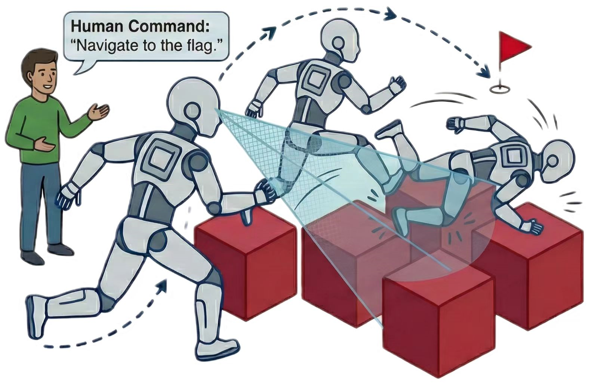}
    \caption{VLA: Task-Oriented \\ (Potential Collision)}
    \label{fig:sub1}
  \end{subfigure}
  \hfill
  \begin{subfigure}[t]{0.49\linewidth}
    \centering
    \includegraphics[width=\linewidth]{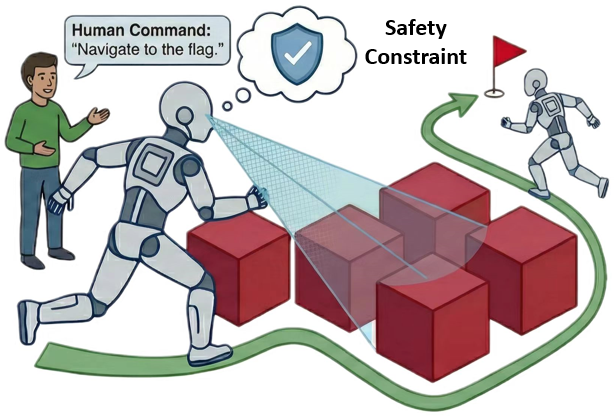}
    \caption{VLSA: Task-Oriented With Safety\\ (Collision-Free)}
    \label{fig:sub2}
  \end{subfigure}

  \caption{Illustrative Comparison of VLA and VLSA Model Behaviors.}\label{figure1}
\end{figure}
Therefore, strict physical safety must be enforced while instruction-following capabilities are maintained. However, directly integrating analytic safety filters like \textit{control barrier functions} (CBFs) into end-to-end VLA models is non-trivial due to two main challenges: the modality gap between raw visual inputs and precise geometric states, and the semantic gap between language instructions and actual collision hazards \cite{ames2019control}. Consequently, a naive integration often leads to overly conservative behaviors, such as treating necessary objects as obstacles, thereby hindering task execution.

To address these challenges, a {\it vision-language-safe action} (VLSA) architecture is introduced. As illustrated in Fig.~\ref{figure1}, conventional VLA models are designed to execute semantic instructions. In contrast, the VLSA framework introduces a  safety constraint layer that dynamically adjusts the original \textit{action} to be a \textit{safe action} while preserving the intended instruction. Following the structure of VLSA, we propose AEGIS (\textbf{A}ction \textbf{E}xecution \textbf{G}uarded by \textbf{I}nvariant \textbf{S}afety) in this paper. Specifically, AEGIS leverages the reasoning power of {\it vision-language models} (VLMs) to identify objects in the scene that may obstruct robotic motion based on natural language instructions and visual observations By incorporating open-set object detection and depth information, semantic-level risks are translated into physical-space avoidance requirements. A  CBF-based \textit{quadratic programming} (QP) solver is then constructed. During inference, nominal \textit{actions} generated by the VLA model are monitored in real-time. The SC layer activates only when potential safety violations are detected, thereby preserving the original task intent while enforcing mathematically proven strict safety guarantees and ensuring safe robotic operation. 

Main contributions are summarized as follows:
\begin{enumerate}


    \item [1)] We propose AEGIS, the first approach that integrates CBFs into VLA models to enforce explicit safety constraints. By introducing a plug-and-play SC layer, AEGIS bridges visual perception and semantic understanding with safety-guaranteed control, without the need for retraining.
    
    \item[2)] We establish SafeLIBERO, a comprehensive safety-critical benchmark derived from the LIBERO dataset, encompassing 32 diverse scenarios and 1600 episodes with various types of obstacles. It provides a  platform for evaluating the safety and robustness of VLA policies in complex environments.
    
    \item[3)] Extensive simulation studies on SafeLIBERO demonstrate that AEGIS achieves superior performance to state-of-the-art baselines, yielding a more than 50\% improvement in obstacle avoidance and a nearly 10\% gain in task success rate. In addition, real-world robotic experiments are conducted to validate the practical applicability of the proposed framework.

\end{enumerate}

\section{Related Work}
\subsection{Vision-Language-Action Models}
VLA models unify visual encoding, language understanding, and action control into an end-to-end framework, enabling robots to generate coherent physical actions from visual observations and natural-language goals \cite{zitkovich2023rt,zhong2025survey,sapkota2025vision}. They bridge the gap between semantic interpretation and embodied control, serving as a key technological pathway toward general-purpose and highly generalizable embodied intelligence. A VLA model maps vision-language inputs to actions via a pipeline of encoding, multimodal fusion, and decoding. Recent progress in VLA research has been driven by advances in spatial reasoning, multimodal alignment, large-scale pretraining, model architecture and inference efficiency \cite{li20253ds,zhou2025chatvla,ge2025vla,kim2024openvla,black2025pi,intelligence2025pi}. 

Despite these impressive capabilities in semantic understanding and task generalization, most existing VLA models overlook the critical aspect of \textit{safety}. Consequently, direct deployment of such {\textit{black-box}} policies in real-world environments remains risky, as they may generate erratic or unsafe trajectories when facing out-of-distribution scenarios \cite{neary2025improving}. To address the issue, recent work SafeVLA \cite{zhang2025safevla} has attempted to incorporate safety considerations by integrating constraints into the training process via reinforcement learning. However, such retraining-based methods require high computational costs and are difficult to apply directly to existing VLA models. Furthermore, relying solely on reinforcement learning optimization typically treats safety as a soft objective (i.e., reward penalties) rather than a hard constraint \cite{hasanzadezonuzy2021learning}. While effectively improving safety levels, these approaches lack explicit, analytic mechanisms to enforce boundary conditions during inference, leaving the robot behavior dependent on probabilistic model outputs rather than grounded physical constraints. Therefore, it is essential to explore a training-free framework that enforces explicit safety constraints during VLA inference.

\subsection{Safety-Guaranteed Control}
Traditional approaches to robot safety typically rely on motion planning algorithms such as RRT* \cite{karaman2011sampling}, or reactive methods like \textit{artificial potential fields} (APF) \cite{khatib1986real}. While effective in classical settings, these methods are unsuitable for VLA models, as overriding the end-to-end semantic actions with a global planner discards the model's intent. Furthermore, these conventional methods often lack rigorous theoretical safety guarantees \cite{singletary2021comparative}. In contrast, CBFs  have emerged as a preferred solution for robot control \cite{xiao2023barriernet}. Acting as a safety filter and typically formulated as a optimization problem, CBFs minimally adjust the nominal action of robots to ensure the forward invariance of a safe set \cite{knoedler2025safety}. The mechanism strictly prevents collisions while preserving the original task behavior to the maximum extent possible.

However, as introduced earlier, integrating CBFs into VLA frameworks presents two practical challenges. First, standard CBFs depend on precise geometric states, such as obstacle positions and shapes, creating a perception gap when dealing with the raw visual inputs of VLA models \cite{liu2025safety}. Second, traditional geometric barriers are semantic-agnostic. All objects within the environment are treated indiscriminately as static obstacles without understanding the task context or assessing their specific safety levels \cite{srinivasan2020synthesis}, which highlights the need for a pipeline that can extract task-relevant geometric primitives from visual data to ground the CBF constraints.

\begin{figure*}[htbp]
  \centering
  \includegraphics[width=0.8\linewidth]{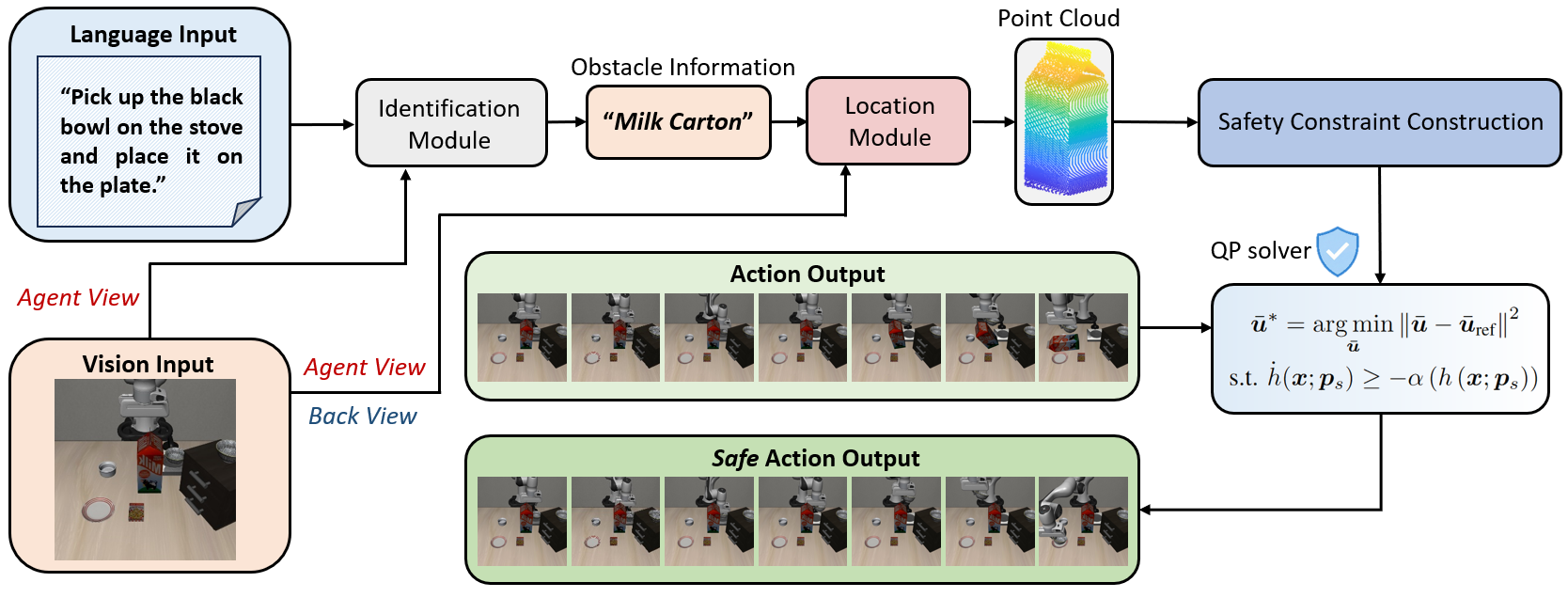}  
  \caption{Workflow of the AEGIS model.}
  \label{fig:example-image}
\end{figure*}
\section{Problem Formulation}
\label{gen_inst}

Consider a robotic manipulation task where a VLA model serves as the high-level policy. Let $o_t$ denote the visual observation and $l$ be the natural language instruction. At each time step, the VLA model predicts a reference action $\bm{u}_{\text{vla}} = [\bm{v}_{\text{vla}}^\top, \bm{\omega}_{\text{vla}}^\top, g_{\text{vla}}]^\top$ based on the inputs, comprising the end-effector translational velocity $\bm{v}_{\text{vla}}$, rotational velocity $\bm{\omega}_{\text{vla}}$, and gripper command $g_{\text{vla}}$. However, since the raw VLA output lacks explicit guarantees for physical safety, we introduce a safety filter based on CBFs.

We model the entire end-effector assembly and task-relevant obstacles as ellipsoids. The safety of the system is encoded by a continuously differentiable function $h(\bm{x})$ \cite{ames2019control}, where $\bm{x}$ consists of the geometric parameters of the ellipsoids and a virtual auxiliary state. The set of safe states is defined as the superlevel set:
\begin{equation}
\mathcal{C} = \{ \bm{x} \mid h(\bm{x}) \geq 0 \}.    
\end{equation}

To strictly enforce the forward invariance of $\mathcal{C}$, we formulate the control problem as a optimization problem, which seeks a safe control input ${\bm{u}}_{\textbf{safe}}$ that minimally deviates from the nominal VLA reference:

\begin{equation}
\label{eq:QP}
\begin{gathered}
{\bm{u}}_{\textbf{safe}}=\underset{{\bm{u}}}{\arg \min }\left\|{\bm{u}}-{\bm{u}}_{\text {vla}}\right\|^2 \\
\text { s.t. } \dot{h}(\bm{x}) \geq-\alpha\left(h\left(\bm{x}\right)\right),
\end{gathered}
\end{equation}
where $\alpha(\cdot)$ is an extended class-$\mathcal{K}_\infty$ function governing the convergence rate to the safe set boundary \cite{drazin1992nonlinear}. The resulting control input $\bm{u}_{\text{safe}}$ is then executed by the robot manipulator.

Our primary objective is to construct a CBF $h$ derived from the language instruction $l$ and visual observation $o_t$, and to solve Eq. \eqref{eq:QP} in real-time for safety-guaranteed control.

\begin{figure}[htbp]
  \centering
  
  \begin{subfigure}[b]{0.49\linewidth}
    \centering
    \includegraphics[width=\linewidth]{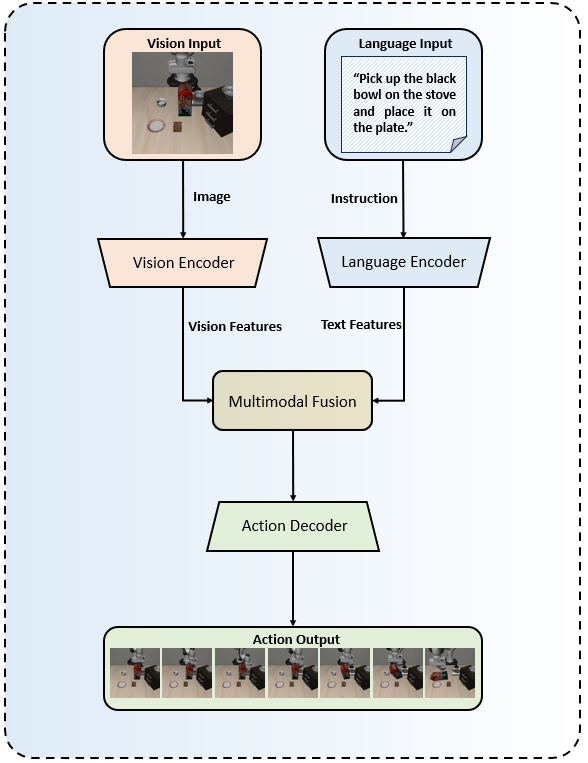}
    \caption{Vision-Language-Action (VLA) Models}
    \label{fig:sub1}
  \end{subfigure}
  \hfill
  \begin{subfigure}[b]{0.49\linewidth}
    \centering
    \includegraphics[width=\linewidth]{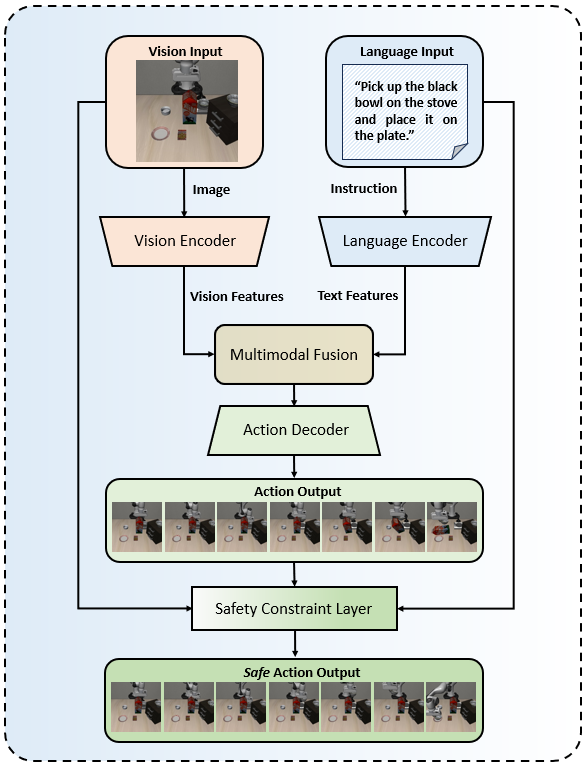}
    \caption{Vision-Language-Safe Action (VLSA) Models}
    \label{fig:sub2}
  \end{subfigure}

  \caption{Functional architecture of VLA and VLSA models.}
  \label{fig:overall}
\end{figure}

\section{VLSA Models and AEGIS Framework}
\label{headings}

\subsection{Main Architecture}
The main architectures of AEGIS and VLSA are shown in Figs.~\ref{fig:example-image} and \ref{fig:overall}, respectively. Similar to most VLA architectures, VLSA uses a visual encoder, a language encoder, multimodal fusion, and an action decoder to generate actions from linguistic and visual inputs. Distinctly, as shown in Fig.~\ref{fig:overall}, VLSA incorporates an additional SC layer positioned after the original VLA action output. The SC layer receives visual features, linguistic features, and the action output from the base VLA model. While conventional VLA models emphasize task completion, they typically overlook safety considerations during execution. In contrast, the SC layer modifies potentially unsafe actions into safe alternatives. If no safety risks are identified, the SC layer output remains identical to the original VLA output, thereby maintaining the model’s baseline performance.

To achieve this, AEGIS implements the SC layer by integrating two interconnected modules: a vision-language-based safety assessment module and an action-driven safety-guaranteed control module. As illustrated in Fig.~\ref{fig:example-image}, the safety assessment module first processes linguistic instructions and visual observations to semantically identify and spatially localize the most critical obstacle in the 3D workspace. This refined spatial information is then fed into the safety-guaranteed control module, which models the collision geometry and leverages a CBF approach to correct potentially unsafe actions via a convex QP solver. The detailed formulations of these two modules are presented in Section \ref{subsec:assessment} and Section \ref{subsec:control}, respectively.

\subsection{Vision-Language-based Safety Assessment Module}
\label{subsec:assessment}
\begin{figure*}[htbp]
  \centering
  \includegraphics[width=0.9\linewidth]{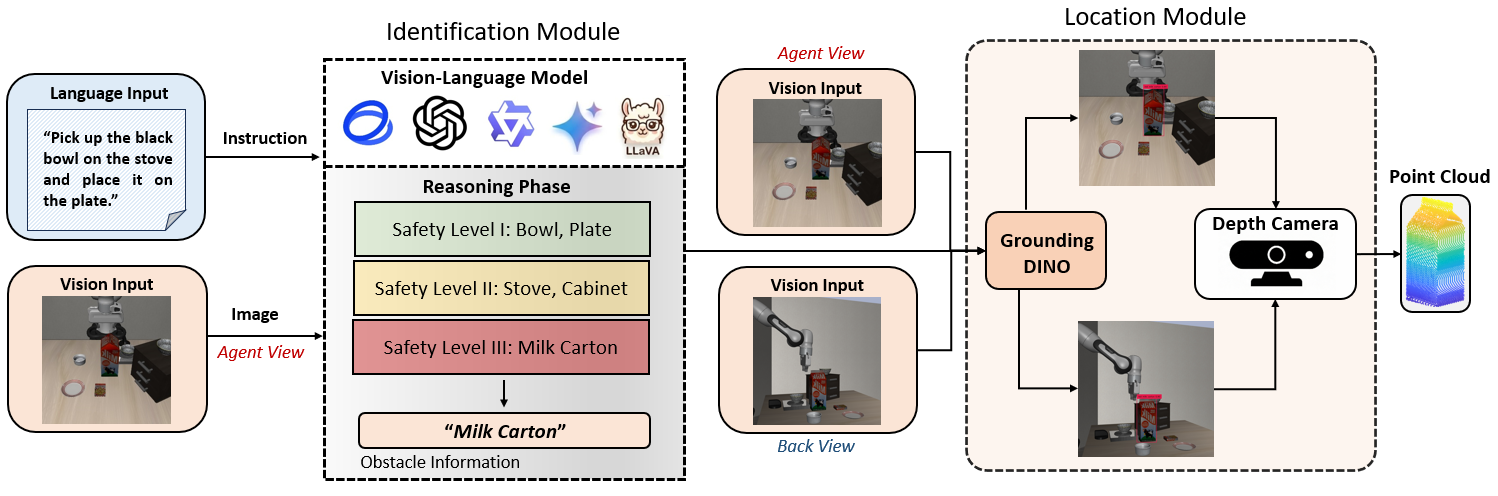}  
  \caption{Pipeline of the vision-language-based safety assessment module.}
  \label{fig:example-image_SA}
\end{figure*}


Safety assessment is foundational for synthesizing safe actions \cite{liu2024online,hu2024cadm+}. As shown in Fig.~\ref{fig:example-image_SA}, our vision-language-based module identifies and localizes potential obstacles that may interfere with the  robot in two main stages: semantic-level identification and precise spatial localization.



In the first stage, a VLM infers potential obstacles by jointly processing a natural language task instruction (e.g., \textit{``Pick up the black bowl...''}) and an agent-view RGB image. To ensure consistent reasoning, the VLM is explicitly prompted to identify the single most critical non-robot obstacle. The constrained output must contain only a uniquely identifiable obstacle name, including its color and type, preferably selected from a predefined list.


The VLM subsequently outputs a concise, uniquely identifiable object name (e.g., \textit{``Red milk carton''}), which directly serves as a textual grounding query for spatial localization. To achieve this, we employ GroundingDINO \cite{liu2024grounding} to process the RGB image alongside the textual query, generating candidate 2D bounding boxes for the identified obstacle. To ensure reliable localization, only the bounding box with the highest confidence score is retained for the  3D projection.

Subsequently, the spatial region corresponding to the selected 2D bounding box is back-projected into 3D space. To obtain a dense and complete 3D representation of the obstacle, we fuse the local point clouds generated from both the agent-view camera and an auxiliary back-view depth sensor. This fusion requires aligning the data into a unified world coordinate system using the standard camera projection model:
\begin{equation}
\left[\begin{array}{c}
X \\
Y \\
Z \\
1
\end{array}\right]=T_{\text {cam }}^{\text {world }} \cdot\left[\begin{array}{c}
K^{-1} \cdot d \cdot\left[\begin{array}{c}
u \\
v \\
1
\end{array}\right] \\
1
\end{array}\right]
\end{equation}
where $(u,v)$ denote the pixel coordinates, $d$ is the depth value, $K$ is the camera intrinsic matrix, and $T_{\text{cam}}^{\text{world}}$ represents the extrinsic transformation from the camera frame to the world frame.

To ensure the reliability of the downstream control module, a rigorous preprocessing pipeline is applied to the fused point cloud. We first impose predefined workspace bounds to filter out irrelevant background points. Next, distance-based outlier removal is performed by discarding the farthest 20\% of points relative to the data centroid. Finally, a clustering operation \cite{khan2014dbscan} isolates the most populous cluster, effectively extracting the main obstacle body.

\subsection{Action-driven Safety-guaranteed Control Module}
\label{subsec:control}

To represent the safety threat efficiently, we adopt a \textit{minimum volume enclosing ellipsoid} (MVEE) to enclose the processed obstacle point cloud $\{\bm{x}_i\}_{i=1}^n$. The obstacle ellipsoid $\mathcal{E}_{ob}$ is defined as:
\begin{equation}
\mathcal{E}_{ob} = \left\{ \bm{x} \in \mathbb{R}^d  \middle|\ (\bm{x} - \bm{c})^\top \bm{R}^\top \bm{Q} \bm{R} (\bm{x} - \bm{c}) \leq 1 \right\},
\end{equation}
where $\bm{c} \in \mathbb{R}^d$ is the center, $\bm{R} \in \mathrm{SO}(d)$ is the orientation matrix, and $\bm{Q} \succ 0$ determines the shape and size. The optimal parameters are obtained by solving the following optimization problem, where minimizing $-\log \det(\bm{Q})$ inherently minimizes the ellipsoid volume:
\begin{equation}
\label{eq:fit}
\begin{aligned}
&\min_{\bm{Q} \succ 0,\ \bm{c},\ \bm{R} \in \mathrm{SO}(d)}   -\log \det(\bm{Q}) \\
s.t. \quad &
\begin{bmatrix}
\bm{Q} & \bm{R}(\bm{x}_i - \bm{c}) \\
(\bm{R}(\bm{x}_i - \bm{c}))^\top & 1
\end{bmatrix} \succeq 0,  i = 1, \dots, n \\
& \bm{R}^\top \bm{R} = \bm{I}, \quad \det(\bm{R}) = 1.
\end{aligned}
\end{equation}

\begin{figure}[htbp]
  \centering
  \includegraphics[width=\linewidth]{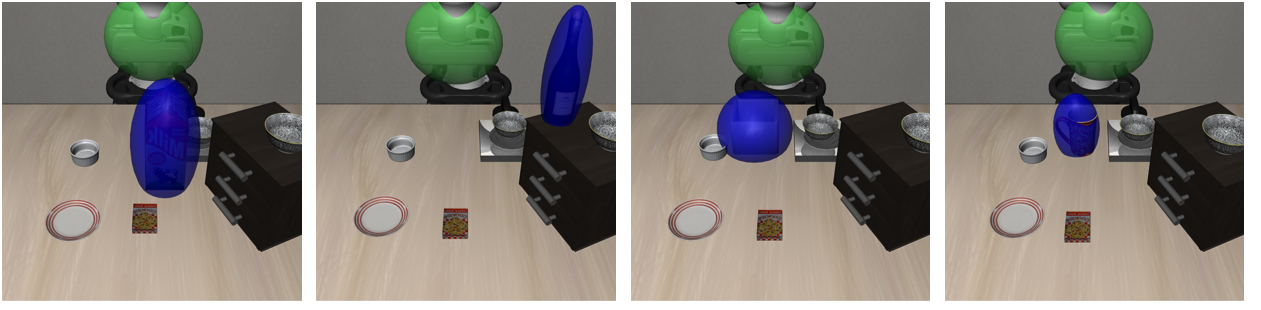}
  \caption{MVEE fitting for the end-effector and the obstacle.}
\end{figure}

Similarly, we model the robot's end-effector as an ellipsoid $\mathcal{E}_{ef}$ with a fixed size matrix $\bm{Q}_{ef}$. Let $\bm{p}_{ef}$ and $\bm{R}_{ef}$ denote the position and rotation matrix of the end-effector, respectively. With a constant structural offset $\Delta \bm{p}$, the ellipsoid center $\bm{p}_{ep}$ is rigidly attached to the end-effector kinematics:
\begin{equation}
\label{eq:p_ep}
\bm{p}_{ep} = \bm{p}_{ef} + \bm{R}_{ef} \Delta \bm{p}.
\end{equation}

To formulate the collision avoidance dynamics, we define an augmented state vector $\bm{x} = [\bm{p}_{ep}, \bm{R}_{ef}, \bm{p}_s]$. Here, $\bm{p}_s$ is a virtual auxiliary state on the unit sphere, mapping to a point $\bm{p}_b$ on the surface of $\mathcal{E}_{ef}$ via:
\begin{equation}
    \bm{p}_b = \bar{\bm{Q}}_{ef}\bm{p}_s + \bm{p}_{ep},
\end{equation}
where $\bar{\bm{Q}}_{ef}=\bm{R}_{ef}\bm{Q}_{ef}\bm{R}_{ef}^\top$. The tangent plane $\mathcal{T}$ at $\bm{p}_b$ is then given by:
\begin{equation}
\mathcal{T}=\left\{\bm{q} \in \mathbb{R}^d \mid \bm{p}_{ef}^{\top} \bar{\bm{Q}}_{ef}^{-1} \bm{q}-\left(1+\bm{p}_{ef}^{\top} \bar{\bm{Q}}_{ef}^{-1} \bm{p}_{ef}\right)=0\right\}.
\end{equation}

Assuming $\mathcal{T}$ initially separates the two ellipsoids, the signed distance $h(\bm{x})$ from the obstacle ellipsoid $\mathcal{E}_{ob}$ to $\mathcal{T}$ is derived as:
\begin{equation}
\label{eq:h}
h(\bm{x})=\frac{-\left\|\bar{\bm{Q}}_{ob} \bar{\bm{Q}}_{ef}^{-1} \bm{p}_{s}\right\|+\left(\bm{p}_{ob}-\bm{p}_{ef}\right)^{\top} \bar{\bm{Q}}_{ef}^{-1} \bm{p}_{s}-1}{\left\|\bar{\bm{Q}}_{ef}^{-1} \bm{p}_{s}\right\|}.
\end{equation}

By maximizing $h$ over $\bm{p}_s$, we obtain the shortest distance between the two ellipsoids. This formulation allows $h(\bm{x})$ to serve as a valid CBF \cite{funada2024collision}. By dynamically controlling the virtual state $\bm{p}_s$, the conservativeness of the avoidance will be reduced. The nominal VLA actions are subsequently adjusted into safe control commands via the QP in Eq.~(\ref{eq:QP}). The theoretical safety guarantee is established in Theorem \ref{thm:safety}.

\begin{theorem}
\label{thm:safety}
Assuming accurate safety assessment, precise point cloud filtering, and complete obstacle representation such that the two generated MVEEs strictly enclose the obstacle and the robot end-effector respectively, the AEGIS framework guarantees that the entire robot end-effector will not collide with the obstacle.
\end{theorem}
\begin{proof}
According to \cite{funada2024collision}, the function $h$ defined in Eq.~(9) constitutes a valid CBF that characterizes the safe superlevel set $\mathcal{C} = \{ \bm{x} \mid h(\bm{x}) \geq 0 \}$, where any augmented state $\bm{x} \in \mathcal{C}$ corresponds to a collision-free configuration between the robot end-effector and the obstacle. Given that the system initializes in a safe configuration (i.e., $h(\bm{x}(t_0)) > 0$), the QP solver strictly enforces the differential constraint $\dot{h} \geq -\alpha(h)$. By Nagumo's Theorem \cite{blanchini1999set}, this condition ensures the forward invariance of $\mathcal{C}$, implying $h(\bm{x}(t)) \geq 0$ and thus $\bm{x}(t) \in \mathcal{C}$ for all $t \geq t_0$. Geometrically, this guarantees that the intersection of the proxy ellipsoids remains empty, i.e., $\mathcal{E}_{ef} \cap \mathcal{E}_{ob} = \varnothing$. Furthermore, based on the strict constraint Eq. (\ref{eq:fit}) where the robot end-effector $\mathcal{A}_{ef}$ and the obstacle $\mathcal{O}_{ob}$ satisfy $\mathcal{A}_{ef} \subseteq \mathcal{E}_{ef}$ and $\mathcal{O}_{ob} \subseteq \mathcal{E}_{ob}$, the disjointness of the supersets necessitates the disjointness of their subsets: $\mathcal{A}_{ef} \cap \mathcal{O}_{ob} \subseteq \mathcal{E}_{ef} \cap \mathcal{E}_{ob} = \varnothing$. Consequently, the physical robot end-effector is guaranteed to remain collision-free.
\end{proof}
\section{Simulation Studies}
\label{sec:experiments}

\subsection{Setup}


\textbf{Benchmark.} The original LIBERO benchmark \cite{liu2023libero} lacks safety-critical scenarios, as objects rarely collide with the robot arm during movement. Therefore, we introduce \mbox{SafeLIBERO}, a comprehensive benchmark for evaluating safety-aware manipulation in complex environments. We select 16 tasks across four LIBERO suites (Spatial, Object, Goal, Long) and introduce a diverse set of everyday objects as obstacles to create two safety-critical scenarios per task: Level I (obstacles placed near the target object) and Level II (obstacles placed directly obstructing the movement path). We randomize object poses across 50 episodes per scenario, resulting in 32 unique scenarios and 1,600 total episodes.
The benchmark tasks are illustrated in Fig.~\ref{fig:benchmark}.

\begin{figure}[htbp] 
    \centering
    \scriptsize 
    \begin{minipage}[t]{0.49\linewidth}
        \centering
        {\tiny \textbf{Pick up the orange juice and place it in the basket}}
        \par\vspace{1pt}
        \begin{subfigure}[b]{0.48\linewidth}
            \includegraphics[width=\linewidth]{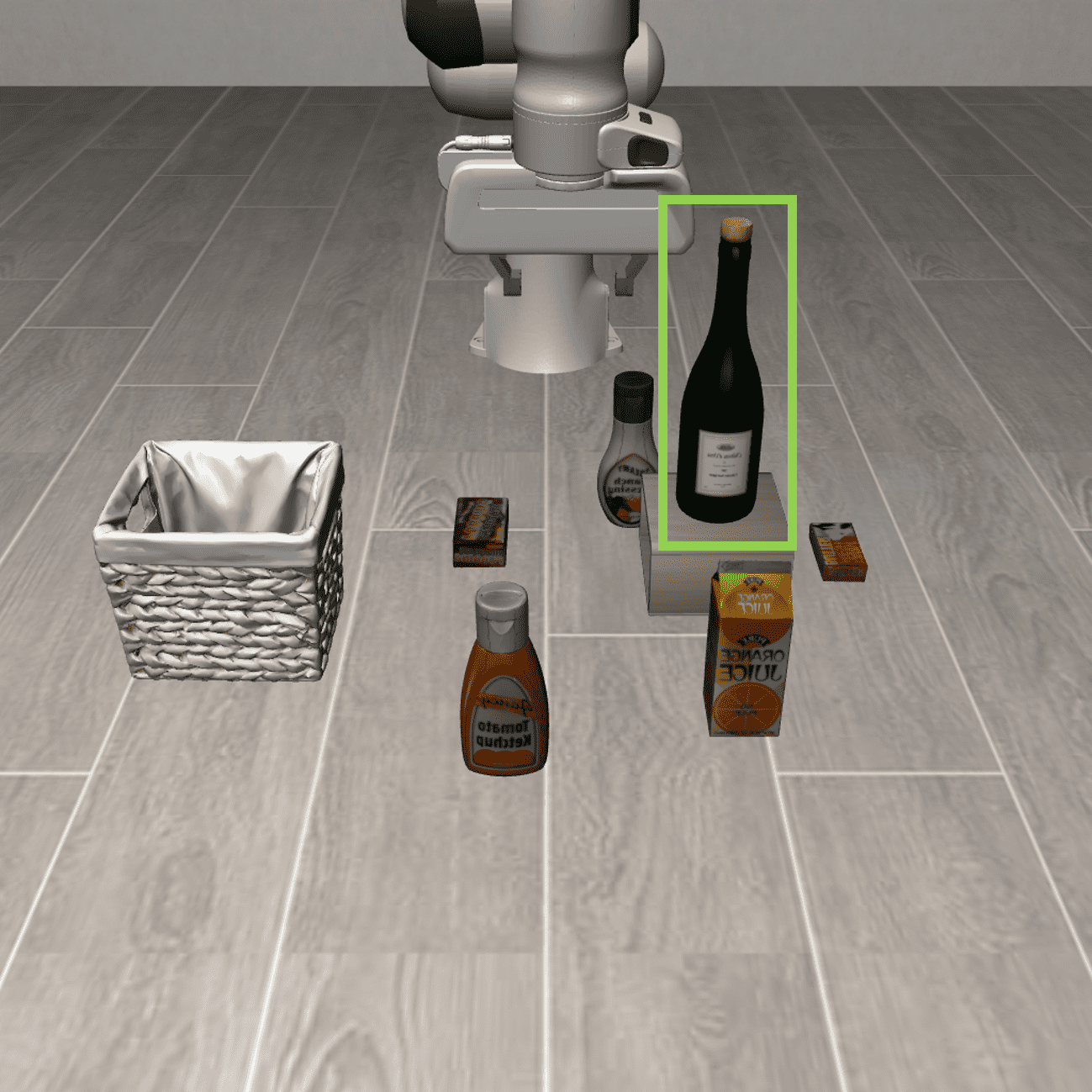}
        \end{subfigure}
        \hfill
        \begin{subfigure}[b]{0.48\linewidth}
            \includegraphics[width=\linewidth]{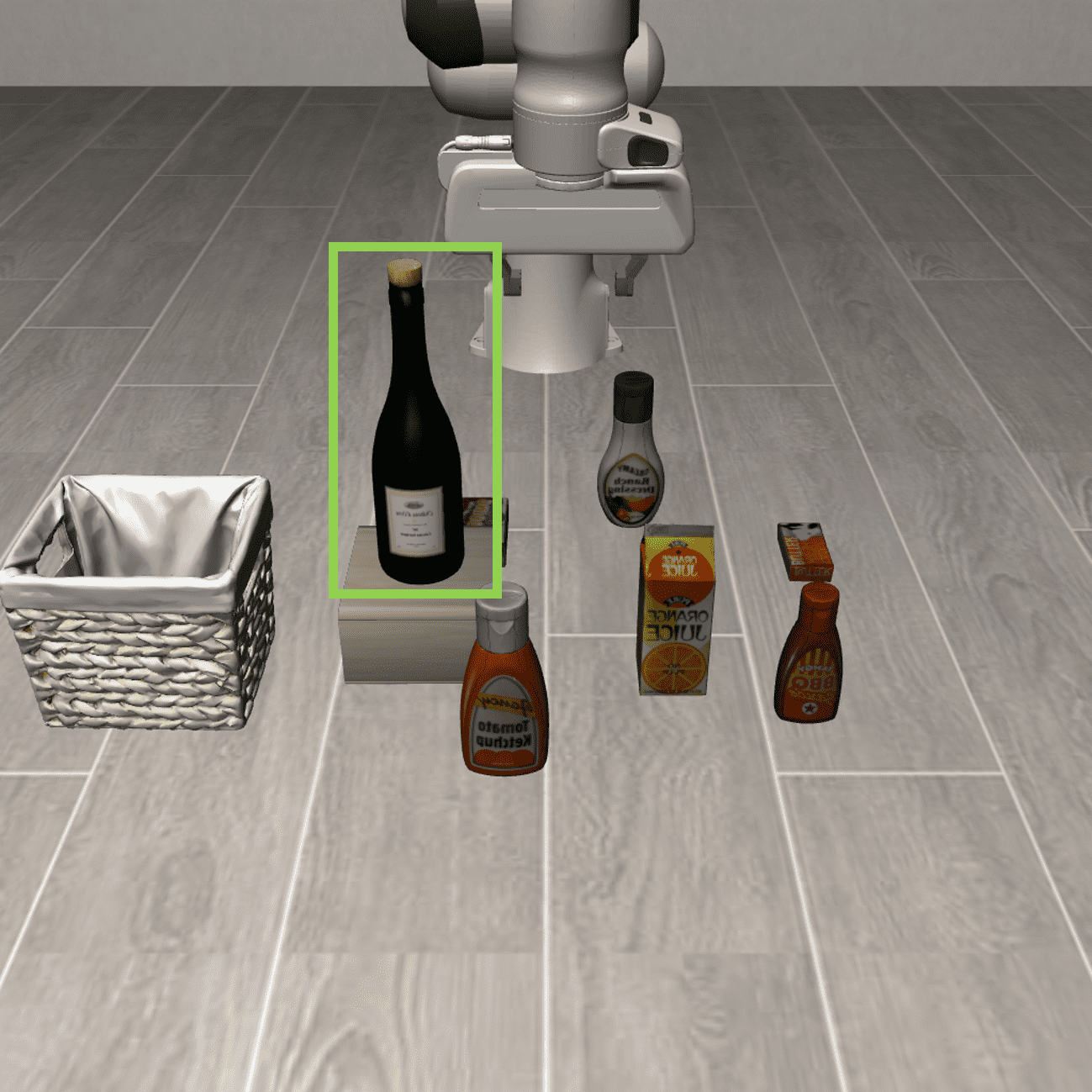}
        \end{subfigure}
        \par
        {\tiny (a) Object}
    \end{minipage}    
    \hfill
    \begin{minipage}[t]{0.49\linewidth}
        \centering
        {\tiny \textbf{Put the bowl on top of the cabinet}}
        \par\vspace{1pt}
        \begin{subfigure}[b]{0.48\linewidth}
            \includegraphics[width=\linewidth]{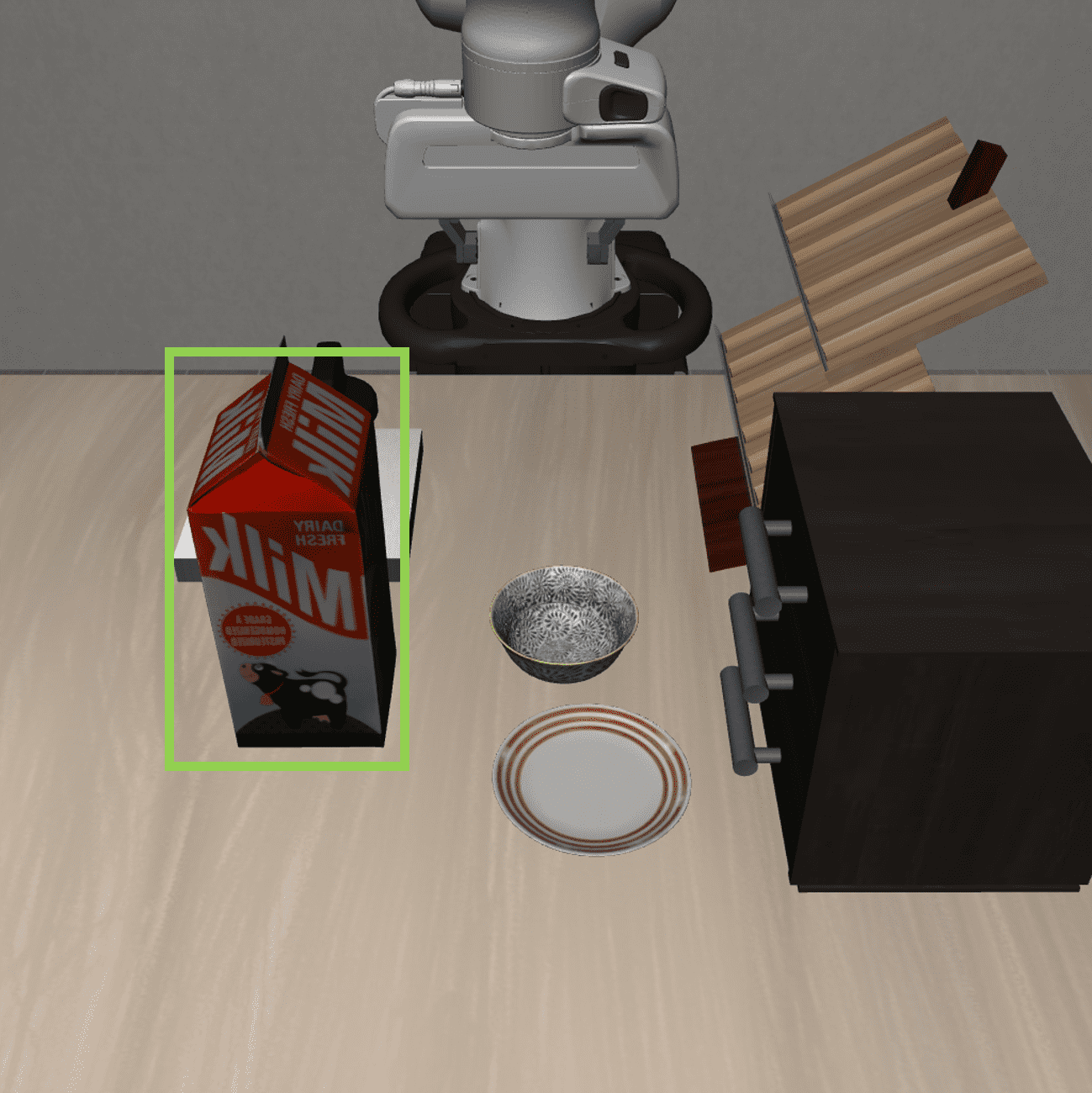}
        \end{subfigure}
        \hfill
        \begin{subfigure}[b]{0.48\linewidth}
            \includegraphics[width=\linewidth]{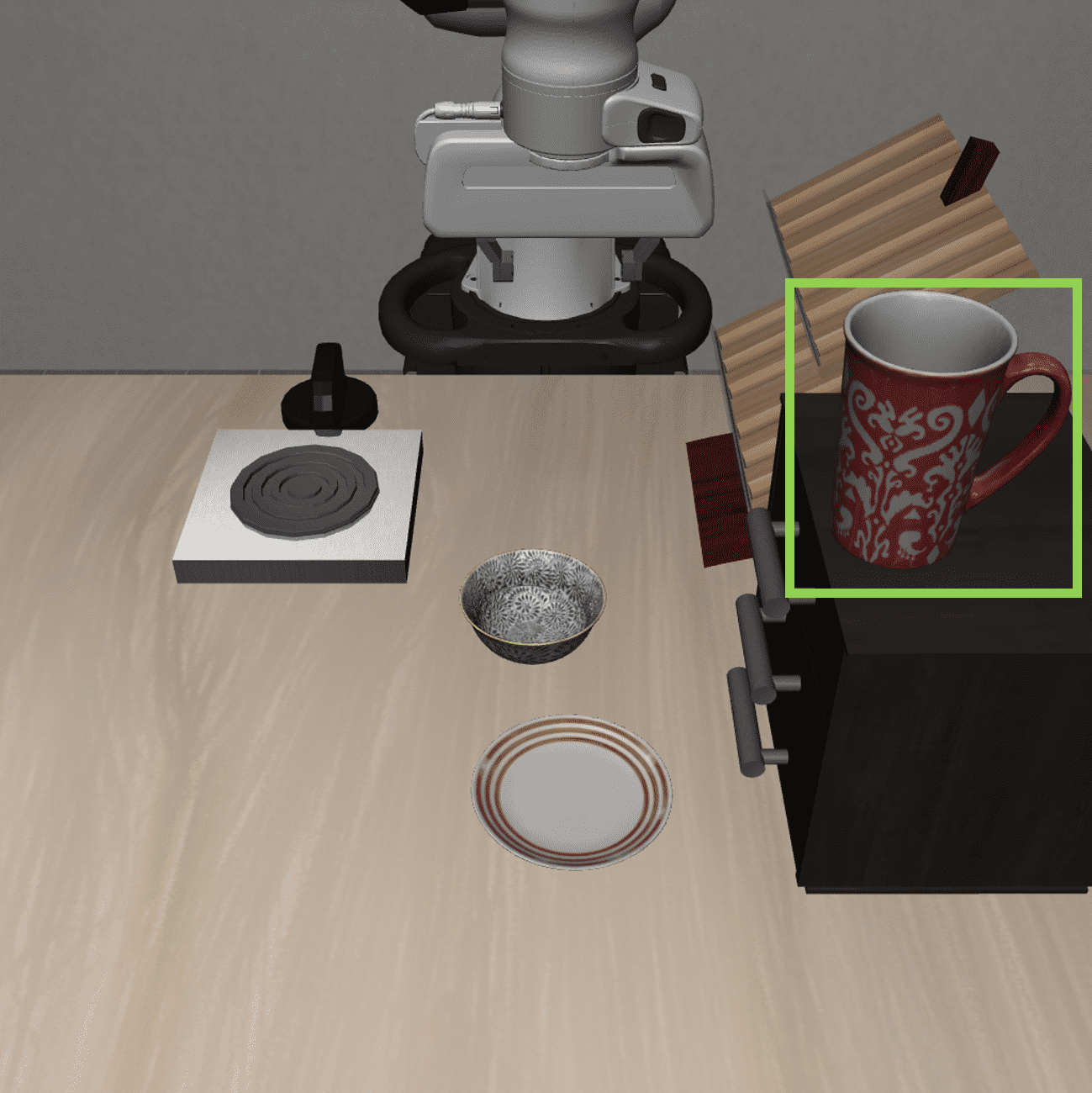}
        \end{subfigure}
        \par
        {\tiny (b) Goal}
    \end{minipage}
    
    \vspace{2mm} 
    

    \begin{minipage}[t]{0.49\linewidth}
        \centering
        {\tiny \textbf{Pick up the black bowl on the stove\\ and place it on the plate}} 
        \par\vspace{1pt}
        \begin{subfigure}[b]{0.49\linewidth}
            \includegraphics[width=\linewidth]{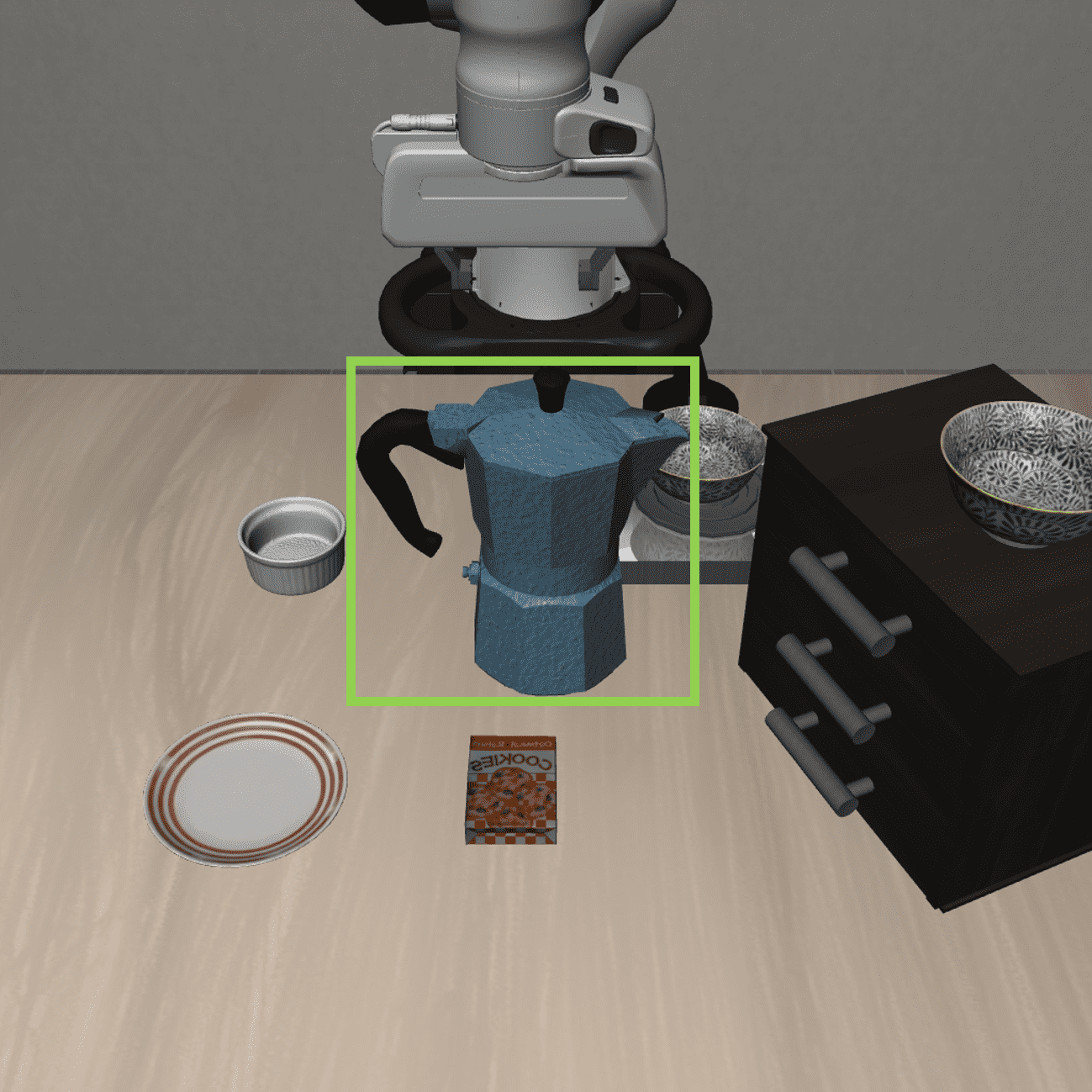}
        \end{subfigure}
        \hfill
        \begin{subfigure}[b]{0.49\linewidth}
            \includegraphics[width=\linewidth]{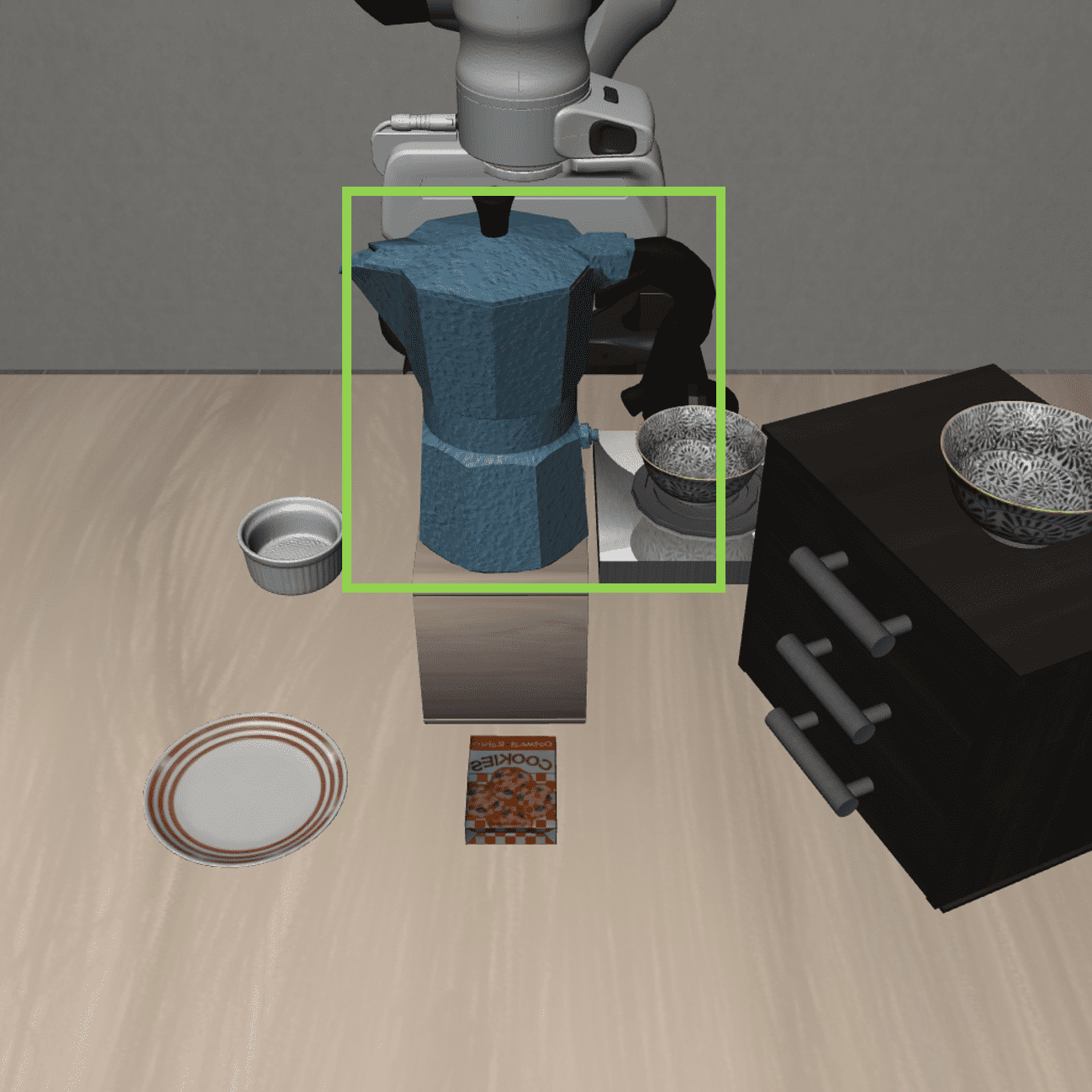}
        \end{subfigure}
        \par
        {\tiny (c) Spatial} 
    \end{minipage}
    \hfill
    \begin{minipage}[t]{0.49\linewidth}
        \centering
        {\tiny \textbf{Put the white mug on the left plate and put the yellow \\
        and white mug on the right plate}}
        \par\vspace{1pt}
        \begin{subfigure}[b]{0.49\linewidth}
            \includegraphics[width=\linewidth]{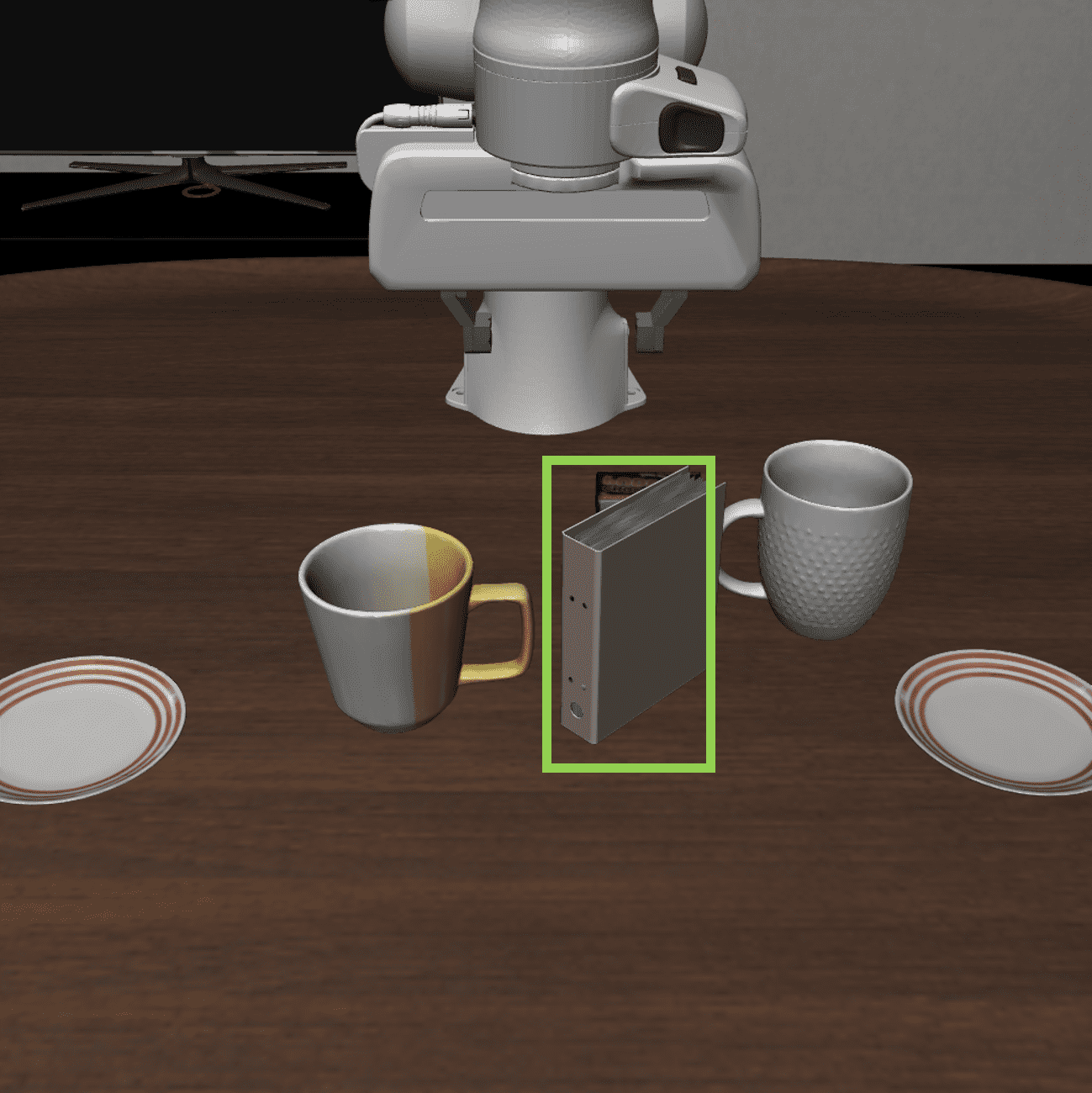}
        \end{subfigure}
        \hfill
        \begin{subfigure}[b]{0.48\linewidth}
            \includegraphics[width=\linewidth]{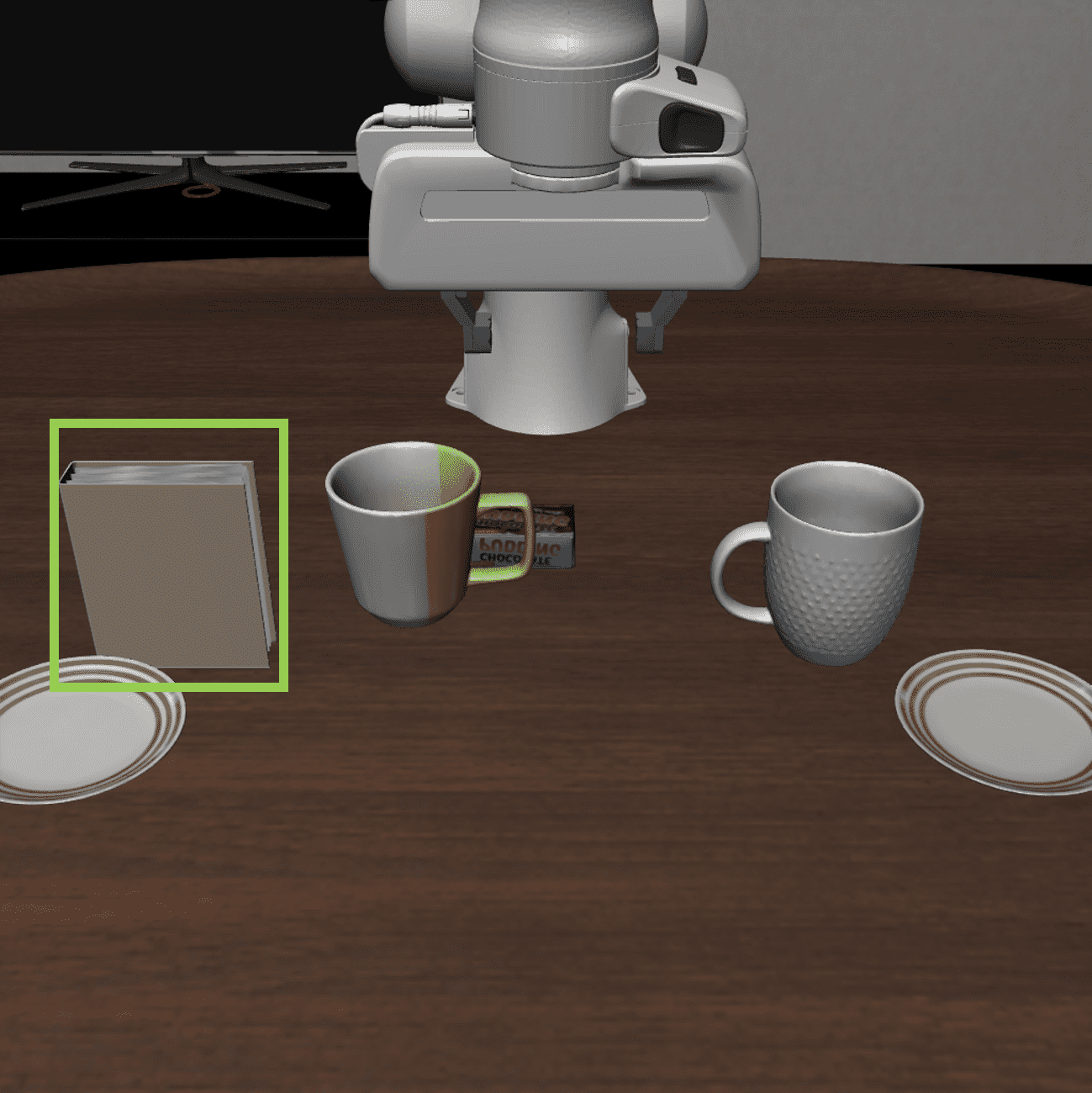}
        \end{subfigure}
        \par
        {\tiny (d) Long}
    \end{minipage}
    
    \caption{Overview of SafeLIBERO benchmark tasks.}
    \label{fig:benchmark}
\end{figure}
\textbf{Baselines.} We utilize the state-of-the-art flow-matching VLA model, $\pi_{0.5}$-LIBERO~\cite{intelligence2025pi}, as our base policy. We compare against the base $\pi_{0.5}$ policy to quantify the direct safety improvements of our plug-and-play module, and against OpenVLA-OFT~\cite{kim2025fine}, a transformer-based VLA adapted via online fine-tuning, to provide a robust cross-architectural baseline. All models are evaluated under identical conditions.

\begin{table*}[htbp]
    \centering
    \renewcommand{\arraystretch}{1.5} 
    \setlength{\tabcolsep}{2pt} 
    
    \caption{\textbf{Quantitative results on the SafeLIBERO benchmark.}}
    \label{tab:main_results}
    
    \resizebox{\textwidth}{!}{%
        \begin{tabular}{l|ccc|ccc|ccc|ccc|ccc|ccc}
            
            \specialrule{0.12em}{1pt}{1pt}
            \specialrule{0.12em}{1pt}{1pt}

            \multirow{3}{*}{\textbf{Task}} & 
            \multicolumn{9}{c|}{\textbf{Translational Action Space Only}} & 
            \multicolumn{9}{c}{\textbf{Full Action Space}} \\
            \cmidrule(lr){2-10} \cmidrule(lr){11-19}

             & \multicolumn{3}{c|}{\textbf{OpenVLA-OFT}$_t$} & \multicolumn{3}{c|}{$\pi_{0.5,t}$} & \multicolumn{3}{c|}{\textbf{Ours}$_t$} 
             & \multicolumn{3}{c|}{\textbf{OpenVLA-OFT}} & \multicolumn{3}{c|}{$\pi_{0.5}$} & \multicolumn{3}{c}{\textbf{Ours}} \\
             
            \cmidrule(lr){2-4} \cmidrule(lr){5-7} \cmidrule(lr){8-10} 
            \cmidrule(lr){11-13} \cmidrule(lr){14-16} \cmidrule(lr){17-19}
            
             & CAR$\uparrow$ & TSR$\uparrow$ & ETS$\downarrow$ 
             & CAR$\uparrow$ & TSR$\uparrow$ & ETS$\downarrow$ 
             & CAR$\uparrow$ & TSR$\uparrow$ & ETS$\downarrow$ 
             & CAR$\uparrow$ & TSR$\uparrow$ & ETS$\downarrow$ 
             & CAR$\uparrow$ & TSR$\uparrow$ & ETS$\downarrow$ 
             & CAR$\uparrow$ & TSR$\uparrow$ & ETS$\downarrow$ \\

            \specialrule{0.10em}{0pt}{0pt}

            
            Spatial & 
            12.8\% & 35.8\% & 238.1 & 
            15.3\% & 59.8\% & 201.7 & 
            \textbf{75.5\%} & \textbf{73.3\%} & \textbf{188.2} &
            8.3\% & 36.5\% & 235.1 &      
            14.0\% & 59.3\% & 199.5 &     
            \textbf{68.0\%} & \textbf{68.5\%} & \textbf{194.9} \\ 

            Goal & 
            25.0\% & 22.5\% & 255.6 & 
            23.8\% & 54.3\% & 210.3 & 
            \textbf{81.5\%} & \textbf{75.3\%} & \textbf{179.6} &
            19.5\% & 24.3\% & 249.0 &     
            20.0\% & 60.0\% & 199.7 &     
            \textbf{76.5\%} & \textbf{82.8\%} & \textbf{166.3} \\

            Object & 
            17.3\% & 29.5\% & 257.4 & 
            23.0\% & 53.8\% & 223.0 & 
            \textbf{74.8\%} & \textbf{80.3\%} & \textbf{201.3} &
            11.5\% & 28.8\% & 259.9 &     
            18.0\% & 57.5\% & 223.5 &     
            \textbf{71.3\%} & \textbf{72.5\%} & \textbf{214.2} \\

            Long & 
            5.5\% & 3.5\% & 541.5 & 
            12.8\% & 35.8\% & 478.0 & 
            \textbf{79.6\%} & \textbf{43.8\%} & 480.1 &
            6.0\% & 15.3\% & 511.7 &      
            16.5\% & \textbf{54.3\%} & \textbf{420.5} &      
            \textbf{59.8\%} & 46.3\% & 455.1 \\

            \specialrule{0.05em}{0pt}{0pt}

            \textbf{Average} & 
            15.1\% & 22.8\% & 323.2 & 
            18.7\% & 50.9\% & 278.2 & 
            \textbf{77.9\%} & \textbf{68.1\%} & \textbf{262.3} &
            5.7\% & 26.2\% & 313.9 &      
            17.1\% & 57.8\% & 260.8 &     
            \textbf{68.9\%} & \textbf{67.5\%} & \textbf{257.6} \\

            \specialrule{0.1em}{1pt}{1pt}
            \specialrule{0.1em}{1pt}{1pt}
            
        \end{tabular}%
    } 
    
    \begin{tablenotes}
        \footnotesize
        \item[] \textbf{Notes}: subscript $t$ denotes policies with \textit{translational} action space only. \textbf{CAR}: \textit{Collision Avoidance Rate}; \textbf{TSR}: \textit{Task Success Rate}; \textbf{ETS}: \textit{Execution Time Steps}. Best results are highlighted in \textbf{bold}.
    \end{tablenotes}
\end{table*}

\begin{figure*}[htbp]
    \centering
    \includegraphics[width=0.85\linewidth]{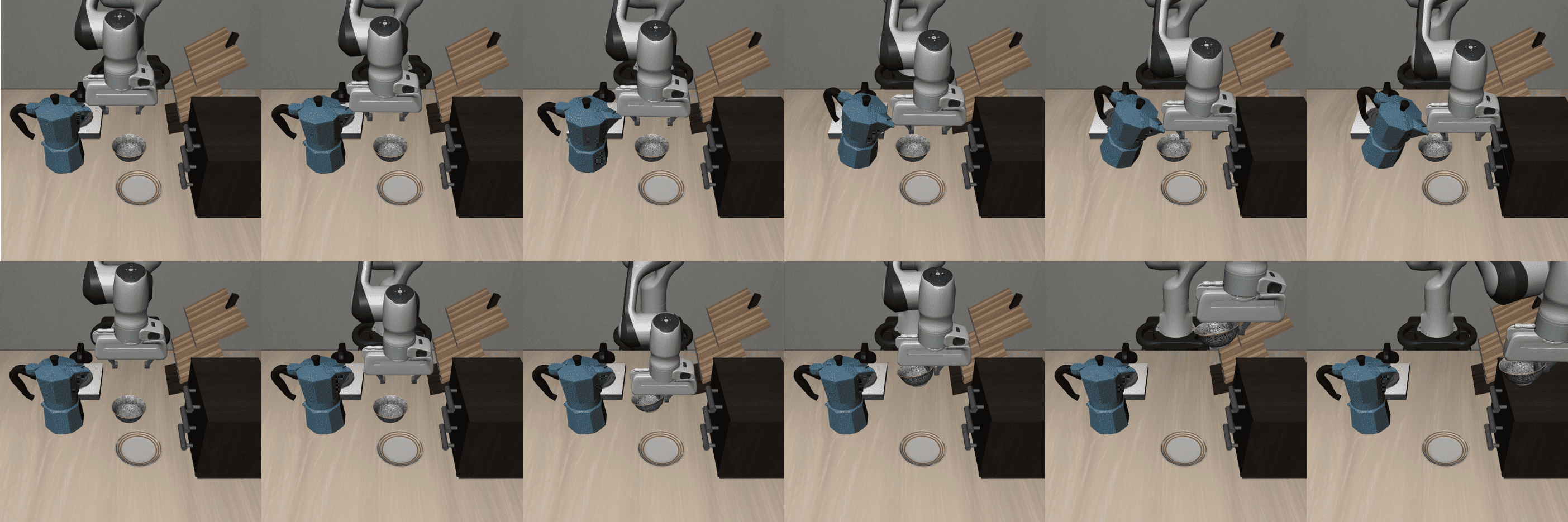} 
    \caption{\textbf{Visual comparison of task execution processes.} The task instruction is \textit{``Put the bowl on top of the cabinet.''} \textbf{Top row:} The baseline policy $\pi_{0.5}$ blindly executes nominal actions, resulting in collisions and task failure. \textbf{Bottom row:} AEGIS dynamically generates collision-free actions to achieve task success. For more qualitative results across other task suites, please refer to the  supplementary video.}
    \label{fig:trajectory_pi05_ours}
\end{figure*}

\textbf{Evaluation Metrics.} We evaluate performance using three metrics: (1) \textit{Collision Avoidance Rate (CAR)}: the percentage of strictly collision-free episodes; (2) \textit{Task Success Rate (TSR)}: the percentage of tasks successfully completed within the time limit, noting that collisions do not trigger early termination; and (3) \textit{Execution Time Steps (ETS)}: the average episode length (including timeouts), where lower values indicate higher efficiency and fewer futile interactions.

\textbf{Settings.} We utilize a Franka Emika Panda robot manipulator controlled via the {OSC\_POSE} interface provided by Robosuite \cite{zhu2020robosuite}, operating at 20 Hz. To provide a comprehensive evaluation, simulation studies are conducted under both translational-only and full action-space settings. The translational setting is explicitly considered because SafeLIBERO tasks predominantly involve top-down manipulation. This configuration reduces action redundancy and enables a focused assessment of positional collision-avoidance performance. In this case, the system dynamics are simplified to the translational kinematic model $\dot{\bm{p}} = 0.2 \bm{u}$, where $\bm{p}$ denotes the end-effector position and $\bm{u}$ is the translational action command. In the full action-space setting, both translational and rotational dynamics are incorporated. The system dynamics are described by $\dot{\bm{p}} = 0.2 \bm{u}_{1:3}$ and $\dot{\bm{\theta}} = 0.2 \bm{u}_{4:6}$, where $\bm{\theta}$ denotes the end-effector orientation and $\bm{u}_{1:6}$ corresponds to the six-dimensional action command. Geometrically, the end-effector is approximated as an MVEE with a size matrix $\bm{Q}_{ef} = \text{diag}(0.06, 0.12, 0.11)$ meters. For safety control, the CBF class-$\mathcal{K}_\infty$ function is set to $\alpha(h) = 10 h$ with a reference control coefficient $k=10$. We employ GLM-4.5V \cite{glm2024chatglm} for vision-language safety assessment. Maximum episode horizons are 300 steps for the Spatial, Goal, and Object suites, and 550 for the Long suite.
\color{black}

\subsection{Results}

\subsubsection{Performance Analysis}
Table~\ref{tab:main_results} presents the quantitative comparison on SafeLIBERO, with more detailed results provided in the supplementary materials\footnote{\nolinkurl{https://github.com/THU-RCSCT/vlsa-aegis/}\linebreak\nolinkurl{blob/main/SupplementaryMaterials.pdf}}. AEGIS shows significant superiority in both safety and task execution. 

Most notably, AEGIS achieves a fourfold increase in the CAR, reaching {77.9\% / 68.9\%} (translational / full action space) compared to the base $\pi_{0.5}$ (18.7\% / 17.1\%) and OpenVLA-OFT (15.1\% / 5.7\%). Importantly, our results indicate that safety is a prerequisite for success in cluttered environments. Baselines frequently suffer from cascading failures—such as knocking over obstacles that subsequently occlude the target objects, particularly in the Level I Object suite. By strictly preventing these disruptive collisions, AEGIS preserves workspace integrity and achieves the highest TSR of {68.1\% / 67.5\%}, significantly outperforming $\pi_{0.5}$ (50.9\% / 57.8\%). Furthermore, AEGIS records the lowest ETS, proving that the safety layer avoids unnecessary extra movements and prevents the robot from getting stuck on obstacles, which frequently slows down the baselines.

\begin{figure}[htbp] 
    \centering
    
    \begin{subfigure}[b]{0.48\linewidth} 
        \centering
        \includegraphics[width=\linewidth]{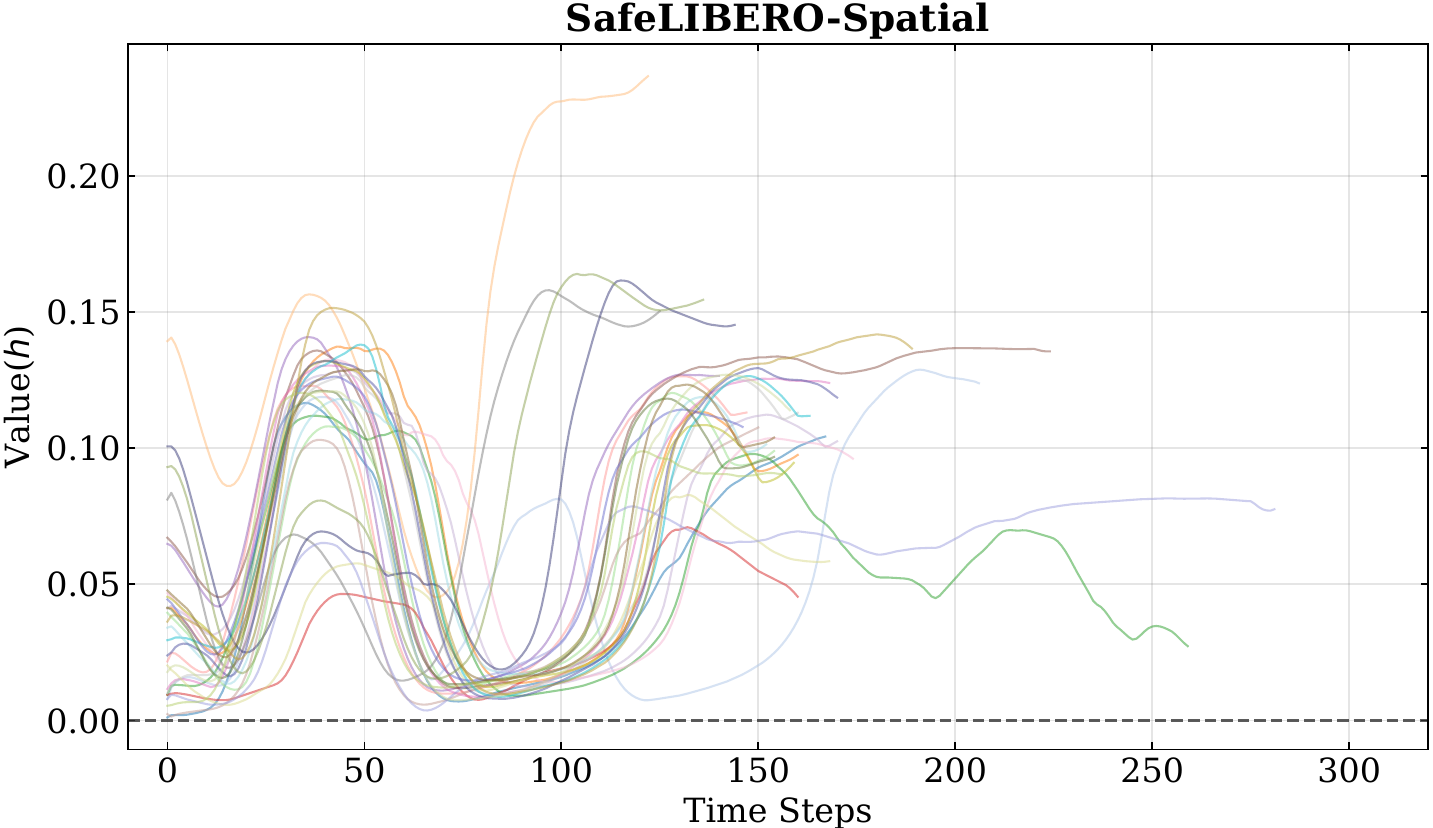}
        \caption{Spatial}
        \label{fig:sub1}
    \end{subfigure}
    \hfill 
    \begin{subfigure}[b]{0.48\linewidth} 
        \centering
        \includegraphics[width=\linewidth]{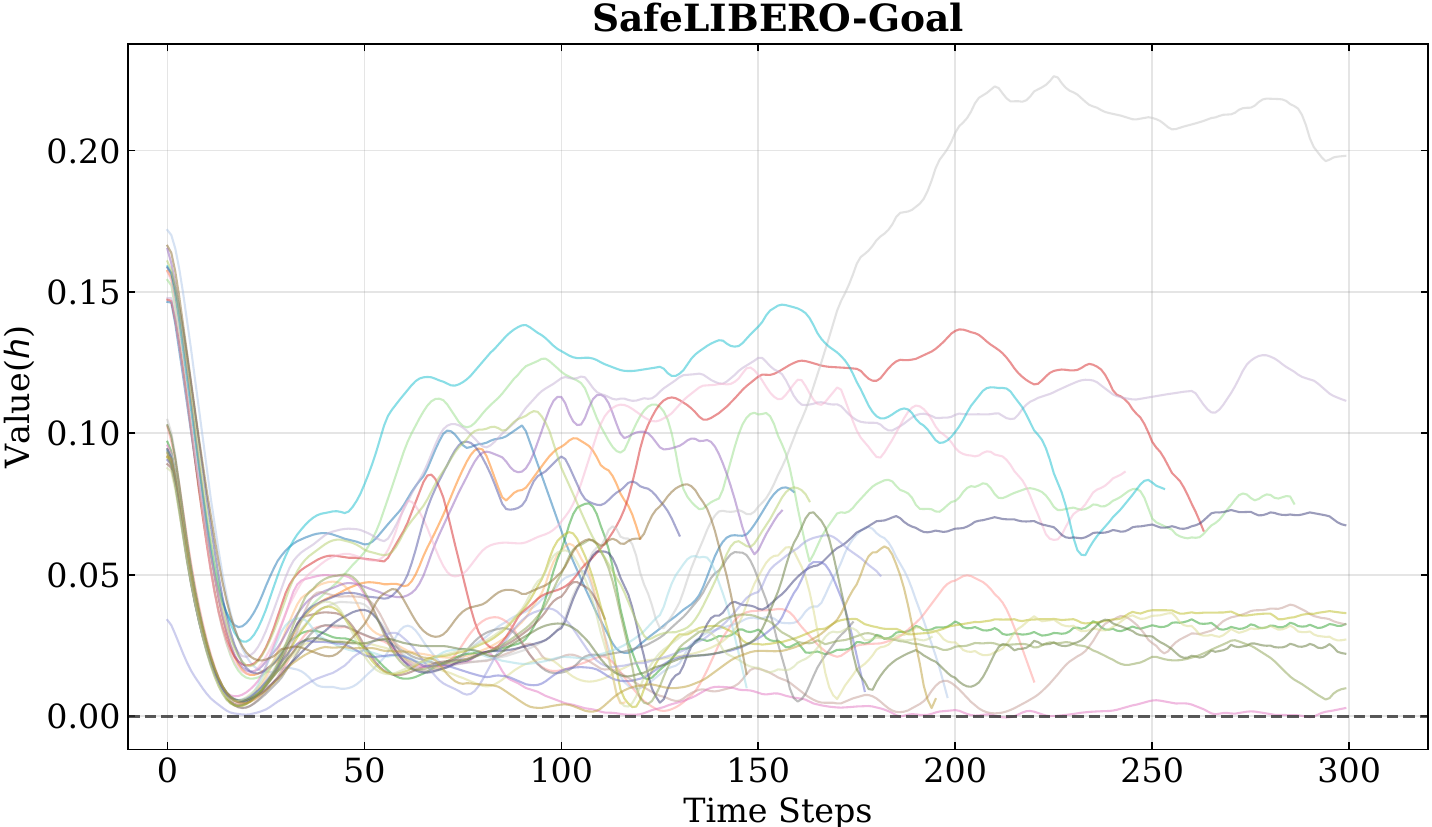}
        \caption{Goal}
        \label{fig:sub2}
    \end{subfigure}
    
    \vspace{1mm} 
    
    \begin{subfigure}[b]{0.48\linewidth}
        \centering
        \includegraphics[width=\linewidth]{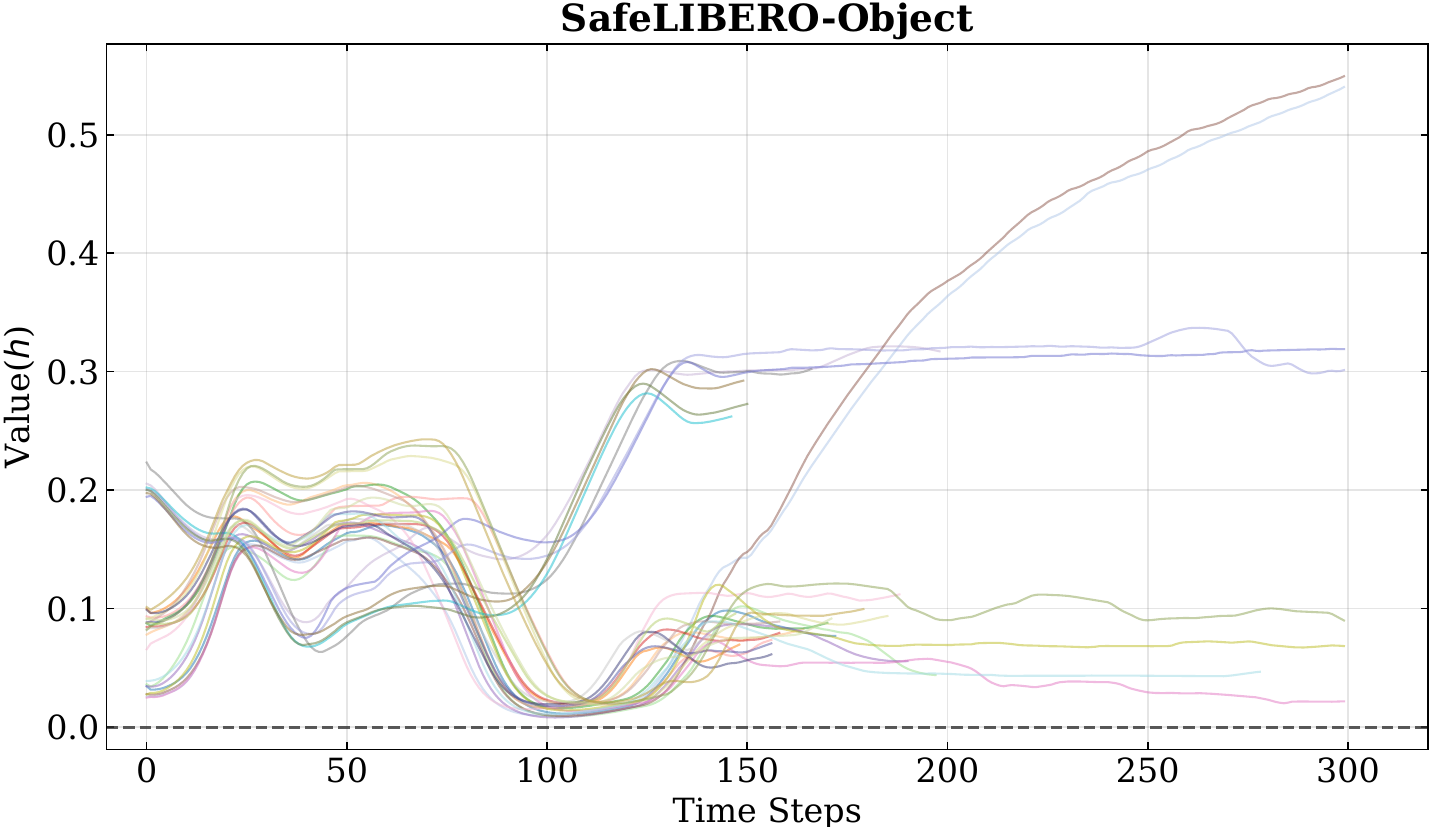}
        \caption{Object}
        \label{fig:sub3}
    \end{subfigure}
    \hfill
    \begin{subfigure}[b]{0.48\linewidth}
        \centering
        \includegraphics[width=\linewidth]{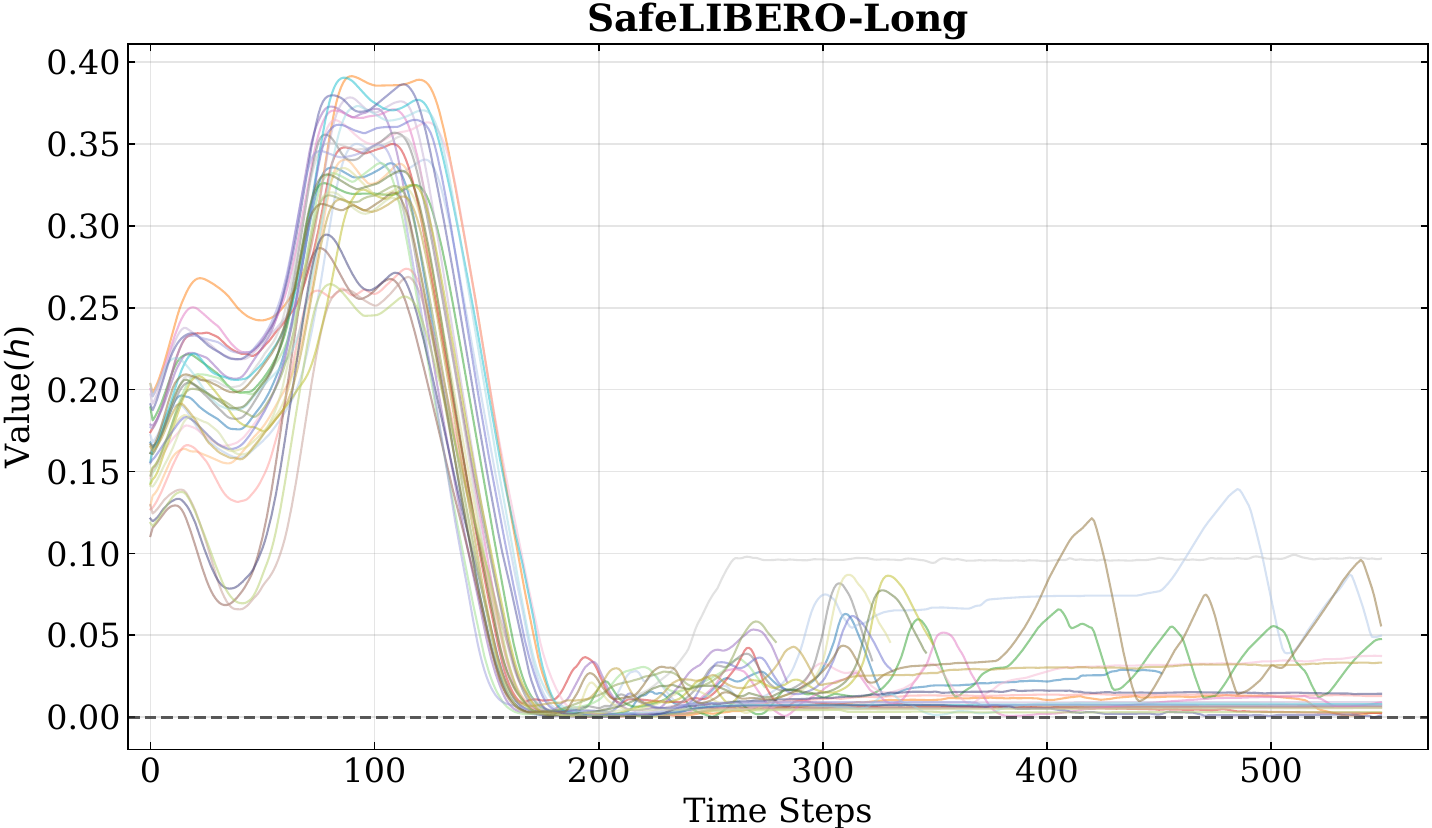}
        \caption{Long}
        \label{fig:sub4}
    \end{subfigure}
    
    \caption{Constraint evolution analysis. Note that in single-column mode, ensure axis labels are legible.}
    \label{fig:cbf_curves}
\end{figure}

\subsubsection{Behavioral and Constraint Analysis}
\begin{figure*}[htbp]
    \centering
    \setlength{\fboxsep}{0pt}    
    \setlength{\fboxrule}{0.5pt} 

    \begin{minipage}[t]{0.49\textwidth} 
        \centering
        {\small \textbf{Task: Put the cup in the bowl.}} 
        \par\vspace{3pt} 
        
        \begin{subfigure}[b]{0.48\linewidth}
            \centering
            \fbox{\includegraphics[width=\linewidth]{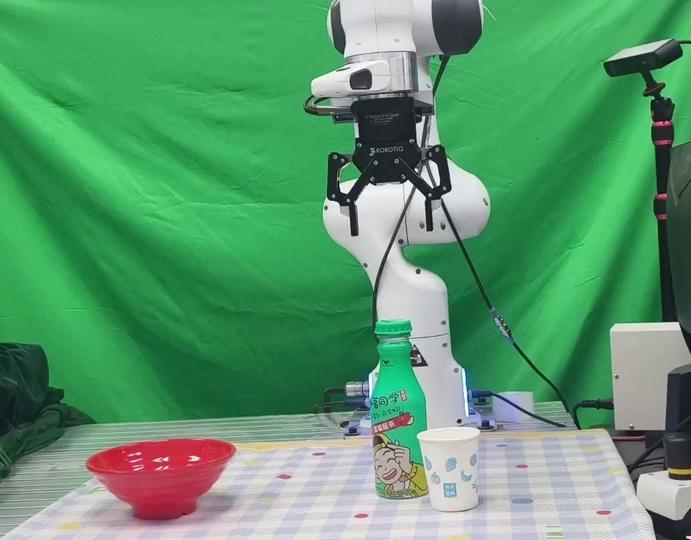}}
            \caption{Level I} \label{fig:cup_i}
        \end{subfigure}
        \hfill
        \begin{subfigure}[b]{0.48\linewidth}
            \centering
            \fbox{\includegraphics[width=\linewidth]{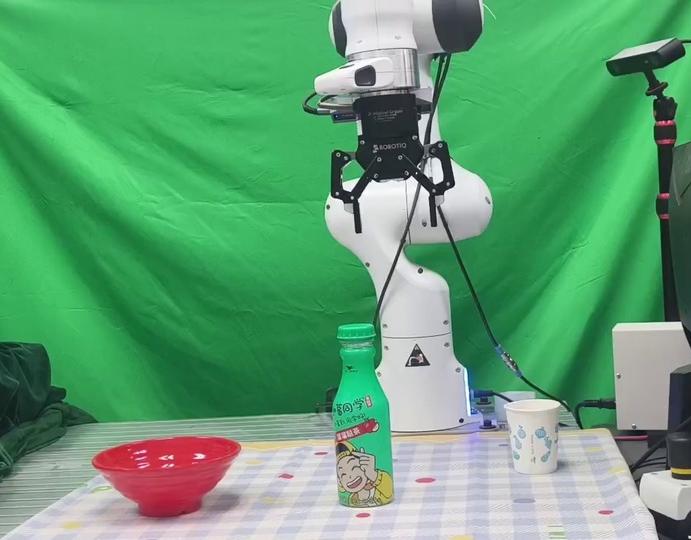}}
            \caption{Level II} \label{fig:cup_ii}
        \end{subfigure}
    \end{minipage}
    \hfill 
    \begin{minipage}[t]{0.49\textwidth}
        \centering
        {\small \textbf{Task: Put the apple in the basket.}}
        \par\vspace{3pt}
        
        \begin{subfigure}[b]{0.48\linewidth}
            \centering
            \fbox{\includegraphics[width=\linewidth]{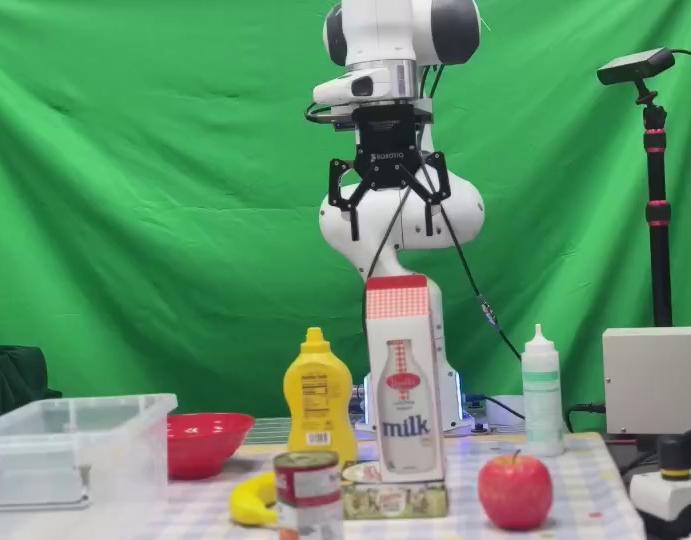}}
            \caption{Level I} \label{fig:apple_i}
        \end{subfigure}
        \hfill
        \begin{subfigure}[b]{0.48\linewidth}
            \centering
            \fbox{\includegraphics[width=\linewidth]{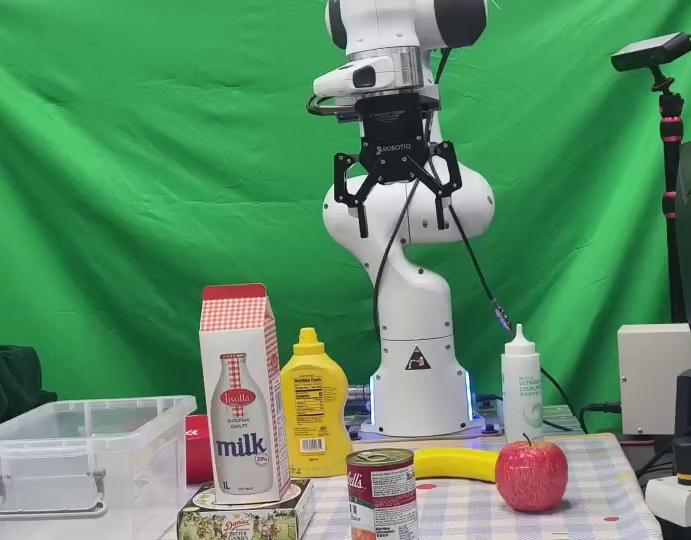}}
            \caption{Level II} \label{fig:apple_ii}
        \end{subfigure}
    \end{minipage}

    \caption{Real-world Experiments. The platform consists of a 7-DoF Franka Emika Panda arm operated in joint velocity control mode, equipped with a Robotiq 2F-85 gripper. Perception is provided by an external ZED 2 stereo camera and a wrist-mounted ZED Mini stereo camera. The VLA policy infers at 15 Hz, while the low-level controller runs at 1 kHz. Experiments encompass two distinct tasks, each evaluated across two levels with varying obstacles.}
    \label{fig:real_world_tasks}
\end{figure*}

\textbf{Qualitative Visualization.} As shown in Fig.~\ref{fig:trajectory_pi05_ours}, baseline policies blindly execute nominal actions, resulting in direct collisions that physically destabilize the environment. In contrast, AEGIS dynamically generates collision-free actions across diverse obstacle geometries while maintaining goal progression. It highlights the effectiveness of our QP formulation, which optimally minimizes deviation from the original VLA action subject to strict safety bounds.

\textbf{Constraint Evolution.} To verify the theoretical validity of our safety layer, Fig.~\ref{fig:cbf_curves} tracks the temporal evolution of the CBF value, $h(\bm{x})$. As the robot approaches an obstacle, $h(\bm{x})$ decreases but consistently remains strictly positive. This empirical evidence confirms that the control solver successfully enforces the forward invariance condition $\dot{h} \geq -\alpha(h)$ at every time step, reliably translating semantic perception into physical safety.

\begin{figure}[htbp]
  \centering
  \includegraphics[width=0.98\linewidth]{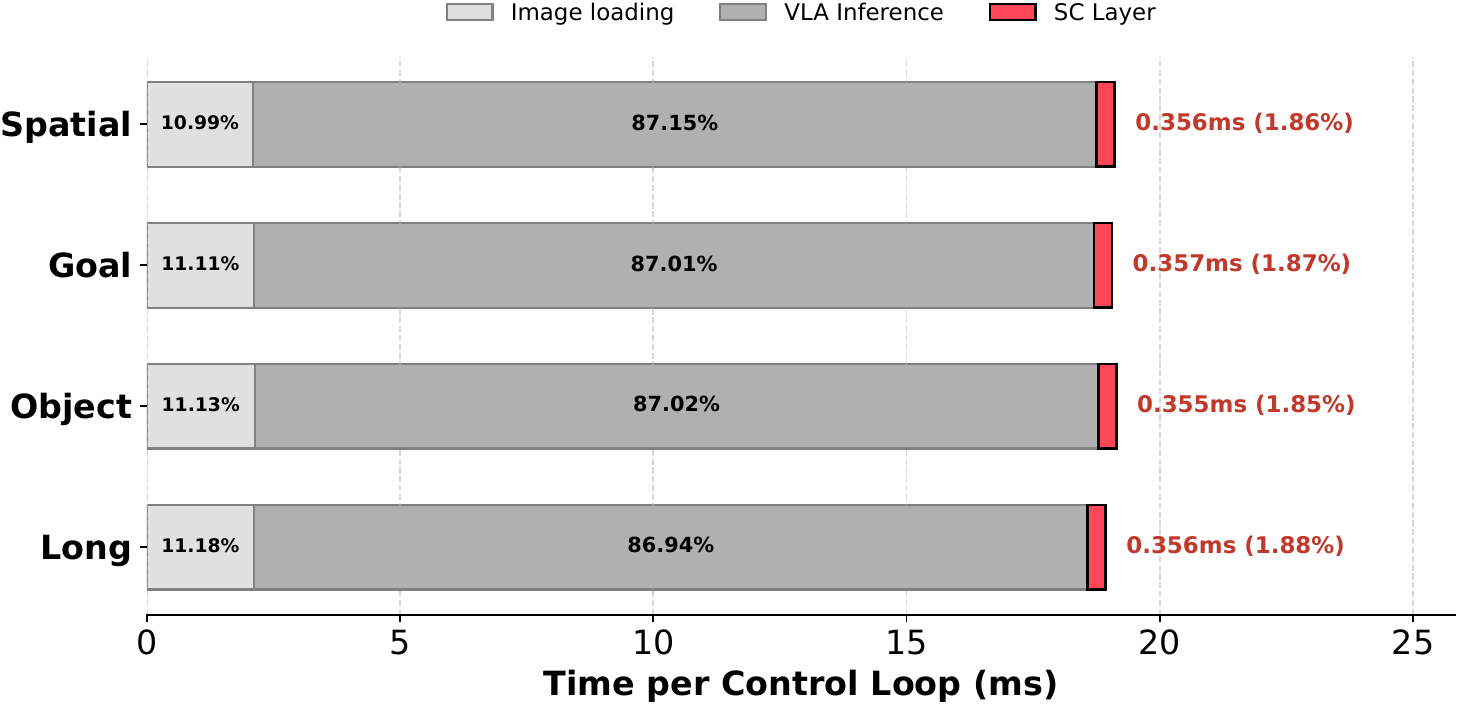}  
  \caption{Time complexity analysis of the proposed approach. SC layer imposes a minimal computational burden ( $<2\%$ of the total cycle time).}
  \label{fig:time_analysis}
\end{figure}
\textbf{Time Complexity.} The underlying SC layer formulates a convex QP with a single linear constraint, effectively acting as a lightweight piecewise intervention that only modifies unsafe nominal actions. As shown in Fig.~\ref{fig:time_analysis}, the SC layer incurs a negligible computation time of merely {0.356 ms} per step on an RTX 4090 GPU setup. This accounts for approximately {1.86\%} of the total cycle latency (roughly 1/47 of the VLA inference time), ensuring the control loop maintains a {20 Hz} real-time execution frequency.

\subsection{Discussion and Limitations}

\textbf{Why is 100\% CAR not Achieved?} 
Ideally, AEGIS guarantees complete collision avoidance for the end-effector. However, residual collisions occur in practice due to limitations in the upstream perception pipeline rather than the control logic. These failures primarily stem from obstacle misidentification, inaccurate spatial grounding, or aggressive point cloud filtering that underestimates the obstacle's geometry. Furthermore, because our current formulation solely constrains the end-effector, unconstrained kinematic links may occasionally collide

\textbf{Safety-Induced Distribution Shift.} 
We observe that although the robot can successfully avoid obstacles, it may subsequently fail to complete the task due to \textit{distribution shift}. Specifically, enforcing safety with AEGIS can drive the system into out-of-distribution states (e.g., higher altitudes rarely seen in the base VLA training data), where the policy may behave erratically and fail to recover toward the goal. Future work should therefore expand training data to better cover these safety-induced out-of-distribution regions.

\color{black}




\section{Real-World Experiments}
\label{sec:real_world_experiments}

To validate the real-world applicability and robustness of our proposed framework, we construct a physical robot platform as shown in Fig.~\ref{fig:real_world_tasks}. We utilize $\pi_{0.5}$-DROID as our base VLA policy. Since our hardware configuration is consistent with that used for DROID dataset collection \cite{khazatsky2024droid}, we deploy $\pi_{0.5}$-DROID in a zero-shot manner without fine-tuning. To align with the simulation studies, we set up two tasks, each comprising two safety levels, as illustrated in Fig.~\ref{fig:real_world_tasks}. In all tasks, the robot is required to accomplish the assigned objective while avoiding collisions.

A comparative evaluation with the $\pi_{0.5}$-DROID policy is conducted, and the results are presented in the supplementary video. In these scenarios, the $\pi_{0.5}$-DROID policy directly collides with the predefined obstacles. In contrast, the proposed method successfully identifies obstacles and leverages point clouds acquired from the external camera to construct CBFs. The nominal actions are subsequently modified to ensure safety, yielding performance that is consistent with the results observed in simulation.
\section{Conclusion}
In this work, a novel AEGIS approach, following VLSA architecture, was designed to bridge the gap between semantic instruction following and physical safety in robotic manipulation. By introducing a plug-and-play SC layer formulated via CBFs, our approach enables existing VLA models to enforce strict safety boundaries with theoretical guarantees without compromising their original task capabilities. We validated our method on the constructed SafeLIBERO benchmark, which covers 32 distinct scenarios with varying spatial complexities. In addition, similar experiments were conducted on a real robotic platform to assess practical applicability. Our extensive experiments demonstrate that VLSA significantly outperforms state-of-the-art baselines. 
It is noteworthy that highly dynamic obstacles and complex whole-body collision avoidance substantially impose stricter safety requirements. Addressing these challenges and extending VLSA to achieve safer and more efficient performance constitute important and compelling directions for future research, and remain
part of our ongoing work.

\section*{ACKNOWLEDGMENT}

The authors thank Xingyu Liu, Yongyi Jia, and Junjie Ding from Tsinghua University for their helpful discussions and support. The authors also acknowledge the use of large language models (e.g., Gemini) for assisting in visualizing concepts and improving the clarity of the manuscript.

\bibliographystyle{IEEEtran}  
\bibliography{reference}       

\end{document}